\pdfoutput=1

\PassOptionsToPackage{hyphens}{url}

\documentclass[11pt, a4paper, onecolumn, copyright]{supply}

\usepackage{cleveref}
\usepackage{textcomp}
\usepackage{epigraph}
\usepackage{tikz}
\usetikzlibrary{positioning, arrows.meta, calc, fit, backgrounds, shapes.geometric, decorations.pathmorphing}
\usepackage{multirow}
\usepackage{array}
\usepackage{makecell}
\usepackage{amsthm}
\usepackage{amsmath}

\newtheorem{proposition}{Proposition}
\newtheorem{definition}{Definition}
\newtheorem{lemma}{Lemma}
\newtheorem{remark}{Remark}

\usepackage[authoryear, sort&compress, round]{natbib}
\usepackage[counter=page, indexonlyfirst=true]{glossaries-extra}
\glsdisablehyper

\newglossarystyle{mytablestyle}{%
  \renewenvironment{theglossary}%
    {\setlength{\LTleft}{0pt}\setlength{\LTright}{\fill}%
     \begin{longtable}{>{\raggedright\arraybackslash}p{0.22\textwidth}%
                       >{\raggedright\arraybackslash}p{0.60\textwidth}%
                       >{\raggedright\arraybackslash}p{0.10\textwidth}}}%
    {\end{longtable}}%
  \renewcommand*{\glossaryheader}{%
    \textbf{Term} & \textbf{Meaning} & \textbf{Page} \\\midrule\endhead}%
  \renewcommand*{\glsgroupheading}[1]{}%
  \renewcommand*{\glossentry}[2]{%
    \textbf{\glsentryname{##1}} & \glsentrydesc{##1} & ##2 \\}%
  \renewcommand*{\subglossentry}[3]{}%
  \renewcommand*{\glsgroupskip}{}%
}
\setglossarystyle{mytablestyle}

\newglossaryentry{grounding}{name={Grounding},
  description={Acquisition, representation, interaction with, and trust in
  verifiable real-world state; the common substrate that simulation,
  self-generation, and pure scaling cannot supply on their own.}}

\newglossaryentry{datawall}{name={Data Wall},
  description={The projected exhaustion of high-quality public training tokens
  in the early 2030s; a frequently cited bottleneck on continued scaling of
  large language and multimodal models.}}

\newglossaryentry{abstractionbarrier}{name={Abstraction Barrier},
  description={The hypothesis that current learning paradigms cannot reliably
  discover new conceptual primitives beyond those latent in human-produced
  text and labels.}}

\newglossaryentry{embodiedbottleneck}{name={Embodied Bottleneck},
  description={The linear physical slowdown that real-world experimentation
  imposes on otherwise digital learning loops, especially in
  manufacturing-dependent settings.}}

\newglossaryentry{multiagenttrust}{name={Multi-Agent Trust},
  description={The requirement that agents in a multi-party economy share
  verifiable signals about identity, capability, inventory, and price before
  reliable cooperation or transaction is possible.}}

\newglossaryentry{supplycertainty}{name={Supply Certainty},
  description={A decomposable property of a real-world supply environment
  consisting of being understandable, comparable, trustworthy, thick, and
  customisable to machine consumers; aggregated by the Supply Certainty
  Index (SCI) over the domain's applicable subset.}}

\newglossaryentry{verifier}{name={Verifier},
  description={A learned or rule-based judge that scores agent or model
  outputs against grounded outcomes (e.g.\ transaction settlement,
  manufacturability, fulfilment), serving as a reward source for downstream
  training; bounded by the Verifier--Goodharting Floor of
  Proposition~\ref{prop:goodhart-floor}.}}

\newglossaryentry{dataflywheel}{name={Data Flywheel},
  description={A closed loop in which deployment outcomes are passed through
  a verifier and fed back as training signal, so that model and product
  improve jointly over time.}}

\newglossaryentry{a2a}{name={Agent-to-Agent (A2A)},
  description={Interaction protocols and market structures in which the
  buying, scheduling, and reconciliation are performed by autonomous agents
  rather than humans.}}

\newglossaryentry{c2m}{name={Customer-to-Manufacturer (C2M)},
  description={A short-loop production model in which buyer-side
  specifications drive small-batch upstream manufacturing, with intermediate
  feedback between demand, design, and factory capacity.}}

\newglossaryentry{cpv}{name={Category-Property-Value (CPV)},
  description={A canonical structured representation of a commercial good
  comprising its taxonomy node, its attributes, and the realised value of
  each attribute; the operational counterpart of a product ontology.}}

\newglossaryentry{effectivecompute}{name={Effective Compute},
  description={The product of nominal hardware compute and algorithmic
  efficiency; growth rate is estimated at roughly an order of magnitude per
  year over the past decade.}}

\newglossaryentry{groundedscaling}{name={Grounded Scaling},
  description={The position, proposed in this paper, that the marginal
  return of compute is gated by the availability of verifiable real-world
  signal; bottlenecks that limit grounding limit scaling.}}

\newglossaryentry{modelcollapse}{name={Model Collapse},
  description={The degradation of generative models when trained recursively
  on their own outputs without sufficient anchoring to real-world
  distribution.}}

\newglossaryentry{deterministicenv}{name={Deterministic Environment},
  description={An environment that returns responses which are stable under
  repeated query, faithfully ranked by relevance, verifiable against
  ground-truth state, and available within bounded latency --- enabling
  reliable multi-step agent execution.}}

\newglossaryentry{sci}{name={Supply Certainty Index (SCI)},
  description={A composite $[0,1]$ score (geometric mean over the
  applicable subset $P_D$ of five property scores), gated by a
  multiplicative deterministic-interface factor $\delta$. The applicable
  subset must be pre-registered per domain to prevent manipulation;
  numerical convention clamps zero scores to an $\varepsilon$-floor.}}

\newglossaryentry{dmm}{name={Determinism Maturity Model (DMM)},
  description={A five-level adoption ladder ($D0$ human-only UI; $D1$
  minimal API; $D2$ SLA-bounded API; $D3$ stable rankings + verifier
  channel; $D4$ full deterministic agent interface with published SCI
  telemetry) introduced in Section~\ref{sec:dmm}, paired with a
  reference architecture and an explicit mapping to existing reliability
  and governance frameworks (EU AI Act, NIST RMF, ISO 42001).}}

\keywords{deterministic agentic AI, grounded scaling, supply certainty, verifier-Goodharting, AGI, ASI}

\correspondingauthor{zuorui.dl@alibaba-inc.com}

\paperurl{}
\reportnumber{}

\uselogo{}

\title{Grounded Scaling: Why Agentic AI Needs Deterministic Environments}

\author[1]{Liang Ding}
\author[1]{Xintong Wang}
\affil[1]{Alibaba Group}

\renewcommand{\today}{}

\begin{document}

\begin{abstract}
Long-chain agent execution fails exponentially in environments designed
for human tolerance: with per-step determinism $\delta < 1$, $k$-step
chain success degrades as $\delta^k$. The AGI-to-ASI scaling
debate~\citep{genewein2026agi} has so far framed progress as a race
between compute growth and a list of frictions (data wall, abstraction
barrier, embodied bottleneck, multi-agent trust); we argue that
environment determinism is a complementary binding axis cutting across
all four, for the broad class of agentic AI tasks whose outcomes are
verifiable economically, physically, or through multi-party settlement.
Three formal results pin down the regime: a Determinism--Efficiency
Bound on chain-task success, a Verifier--Goodharting Floor on flywheel
ceilings under imperfect rewards, and a convergence condition for
environment-side skill evolution. We operationalise the framework as a
Supply Certainty Index (SCI) over five measurable properties, a
five-level Determinism Maturity Model (DMM) as adoption ladder, and a
falsifiable open-question programme (OQ1--OQ5) with explicit null
results that would force retraction. The position is platform-agnostic.
We engage three competing positions: sim-to-real sufficiency, alignment
sufficiency, and AI-as-normal-technology.
\end{abstract}

\maketitle

\epigraph{\itshape ``An ant, viewed as a behaving system, is quite simple.
The apparent complexity of its behavior over time is largely a reflection
of the complexity of the environment in which it finds itself.''}{Herbert
A.\ Simon, \textit{The Sciences of the Artificial} (1969)}

\setcounter{tocdepth}{1}
{%
  \footnotesize
  \setlength{\parskip}{0pt plus 0pt minus 0pt}%
  \makeatletter
  \renewcommand*\l@section{\@dottedtocline{1}{0pt}{1.8em}}%
  \makeatother
  \tableofcontents
}

\section{Introduction and Position Statement}\label{sec:intro}

\subsection*{The determinism gap}

The most influential recent surveys of the path from artificial general
intelligence (AGI) to artificial superintelligence (ASI) frame the open
question as a race~\citep{genewein2026agi, morris2024levels,
bengio2025international}. On one side stands the growth rate of
\gls{effectivecompute} --- the product of nominal hardware progress, capital
investment, and algorithmic
efficiency~\citep{Sevilla2022ComputeTrends, ho2024algorithmic,
hernandez2020scaling, ding2023efficiencyspectrum}, estimated at roughly
$10\times$ per year. On the other stand a set of frictions: a coming
\gls{datawall}~\citep{Villalobos2024WillRunOut, shumailov2024ai}, an
\gls{abstractionbarrier} that current learners may not be able to
cross~\citep{chollet2019measure, ortega2021shaking}, an
\gls{embodiedbottleneck} that constrains recursive self-improvement to
real-world clock time~\citep{lawrence2024atomic}, and a requirement for
\gls{multiagenttrust} as the building block of virtual agent
economies~\citep{tomasev2025virtual, drexler2019reframing}. We do not
dispute that compute and frictions both matter. We dispute that they are on
the same axis.

A parallel but under-discussed factor cuts across all four frictions:
\emph{environmental determinism}. Current agent environments --- web
platforms, enterprise APIs, consumer applications --- were designed for
human tolerance. Search engines intentionally shuffle top results for
diversity; recommendation systems inject exploration noise; response
latencies vary by orders of magnitude; session state makes identical
queries return different answers. For a human user, this non-determinism
is tolerable or even desirable. For an autonomous agent executing a
multi-step task chain, each non-deterministic step compounds failure
probability exponentially.

\begin{definition}[Deterministic Agentic Environment]\label{def:det-env}
An environment $\mathcal{E}$ is \emph{deterministic for agentic
consumption} if, for any well-formed intent $q$, the environment
returns responses satisfying four conditions:
\begin{enumerate}
  \item[\textbf{(D1)}] \emph{Stability}: repeated queries with identical
        parameters yield consistent results within a declared staleness
        bound (covering session-conditioned personalisation as a
        special case of failed stability);
  \item[\textbf{(D2)}] \emph{Faithful ranking}: results are ordered by
        relevance to the stated intent without adversarial reordering
        or exploration injection (a relevance-ordering condition that
        strengthens, not duplicates, (D1));
  \item[\textbf{(D3)}] \emph{Verifiability}: each returned item can be
        checked against a ground-truth state (inventory, price,
        certification, availability) via an independent verification
        channel;
  \item[\textbf{(D4)}] \emph{Bounded latency}: response time is bounded
        by a declared SLA, enabling agents to plan execution within
        deterministic time budgets.
\end{enumerate}
A \gls{deterministicenv} enables reliable multi-step execution; a
stochastic environment designed for human tolerance degrades agent
performance exponentially with chain length
(Proposition~\ref{prop:det-bound}).
\end{definition}

\paragraph{Motivating example: long-chain procurement.}
Consider an autonomous agent executing a procurement task across any
supply platform: find candidate suppliers $\to$ compare specifications
and pricing $\to$ inquire about customisation $\to$ negotiate terms
$\to$ pay $\to$ verify fulfilment.
This is a six-step chain. If each step has independent determinism
quality $\delta < 1$ (e.g., the search engine occasionally returns
stale inventory, the pricing API has variable latency that triggers
timeouts, the comparison module personalises rankings), the chain-success
probability degrades as $\delta^6$. At $\delta = 0.9$ per step, chain
success is only $53\%$; at $\delta = 0.8$, it falls to $26\%$. In
practice, any blockage at any step forces the agent to reroute ---
switching vendors, switching platforms, or abandoning the task entirely.
This rerouting behaviour is universal: it occurs in e-commerce, logistics,
healthcare procurement, agricultural commodity trading, and financial
settlement. The determinism of the environment is not a convenience; it is
a \emph{prerequisite for scalable agentic execution}.

\subsection*{Why now: the 2025--2026 agent inflection}\label{sec:why-now}

We do not claim a sharp 2025--2026 discontinuity --- each of the three
developments below has direct 2023--2024 antecedents. Rather, they
constitute a \emph{continuation} that crosses a threshold of practical
salience: the consequences of environmental non-determinism became
operationally measurable, not just theoretically anticipated.

\paragraph{Operator-class deployed agents.} Major laboratories shipped
agents that act on live, human-optimised web and desktop environments at
production scale: Anthropic's Computer
Use~\citep{anthropic2024computeruse}, OpenAI's
Operator~\citep{openai2025operator}, and a generation of open-source
followers operate on the same stochastic surfaces our argument
identifies. Independent evaluations on OSWorld and $\tau$-bench show
that long-horizon success rates remain well below single-turn
ones~\citep{xie2024osworld, yao2024taubench, jimenez2024swebench},
exactly the failure profile our Proposition~\ref{prop:det-bound}
predicts.

\paragraph{Agent-to-agent commerce protocols.} The standardisation
substrate for agent-mediated transactions has moved from research to
deployment: Anthropic's MCP~\citep{anthropic2024mcp}, OpenAI function
calling~\citep{openai2024function}, Google's Agent2Agent
protocol~\citep{google2025a2a}, and payment-side agent SDKs from Stripe
and others~\citep{stripe2025agenttoolkit}. The bottleneck has shifted
from \emph{can agents talk} to \emph{can they transact reliably on
non-deterministic backends}.

\paragraph{Self-evolving agent stacks.} Skill libraries, tool graphs,
and persistent context now evolve from deployed
trajectories~\citep{yang2025skillopt, feng2026searl, zhang2025ace,
gao2025selfevolvingsurvey, zhang2025dgm, sakana2024aiscientistv2},
making the verifier signal that gates those updates the binding
constraint on environment-side learning quality (formalised below as
Proposition~\ref{prop:gsec-formal}).

\subsection*{Position}

The frictions above are grounding problems, not compute problems.
We use ``grounding'' technically: the verified supply of real-world
state as training and execution signal for agentic systems, where
verifiability comes from an independent ground-truth channel (a payment
cleared, a shipment arrived, a specification was met). This sense is
distinct from symbol grounding~\citep{harnad1990symbol}, visual
grounding~\citep{plummer2015flickr30k}, embodied grounding in robotics,
or perceptual-symbol-systems grounding~\citep{barsalou1999perceptual,
lake2017building}. The question is not how to learn human-like
concepts but how to deploy concept-using agents into real environments
reliably.

The policy lever that most concentrates leverage on AGI$\to$ASI
progress is therefore the construction of substrates that supply
verifiable real-world signal deterministically. Our central claim:
one such substrate already exists at industrial scale in economically
self-sustaining commercial supply environments (B2B sourcing, B2C
retail, pharmaceutical supply chains, agricultural commodity exchanges,
automotive supplier networks). Such environments are structurally
distinct from simulation and self-generation along dimensions
discussed in Section~\ref{sec:supply-privilege}, and that distinctness
translates into measurable advantages on each of the four bottlenecks
\emph{when the environment satisfies Definition~\ref{def:det-env}}.
The position is platform-agnostic. The claim is bounded: a binding axis
for tasks whose verification is naturally economic, physical, or
multi-party (Section~\ref{sec:counter}).

\subsection*{Falsifiable strong form}

We treat the position as a scientific claim, not a slogan, and commit to
a falsifiable strong form:

\begin{takeawaybox}{Box 1 $\vert$ Falsifiable strong form (cf.\ Section~\ref{sec:research-agenda}).}
If the bottlenecks named above are indeed grounding problems, then real
supply environments that are simultaneously \emph{economically
self-sustaining}, \emph{verifier-equipped}, and \emph{deterministic} (in
the sense of Definition~\ref{def:det-env}) should systematically
outperform pure simulation and pure self-generated data in crossing the
data wall, the abstraction barrier, and the embodied bottleneck. This
advantage must be measurable in sample efficiency, in time-to-onset of
recursive degeneration, and in the rate at which novel concepts beyond
human ontology are discoverable. The advantage is in principle refutable
by the experiments OQ1--OQ5 of Section~\ref{sec:research-agenda}.
\end{takeawaybox}

\subsection*{Roadmap}

The paper extends the AGI$\to$ASI position
literature~\citep{genewein2026agi, morris2024levels, bostrom2014,
sutton2019bitter, narayanan2024aiasnormal} with three formal results
($\delta$, $\varepsilon$, $\delta_{\min}$), one composite measurement
(SCI), one adoption ladder (DMM), and an investable, falsifiable
research agenda.

Section~\ref{sec:background} restates the AGI$\to$ASI bottleneck
landscape and introduces the grounded-scaling lens.
Section~\ref{sec:grounding-thesis} derives the grounding common
denominator and states the Determinism--Efficiency Bound and the
Verifier--Goodharting Floor.
Section~\ref{sec:supply-privilege} argues why supply environments
satisfying (D1)--(D4) are structurally privileged.
Section~\ref{sec:five} decomposes supply certainty and constructs the SCI.
Section~\ref{sec:flywheel} treats data-flywheel sustainability and
states the Grounded Self-Evolution Convergence Condition.
Section~\ref{sec:agent-econ} treats multi-agent scaling under supply
trust and determinism, and introduces the DMM.
Section~\ref{sec:research-agenda} states the open-question programme.
Sections~\ref{sec:counter}--\ref{sec:conclusion} close with
counterarguments, competing positions, and conclusion.

\section{Background: Pathways, Bottlenecks, and the Grounding Lens}\label{sec:background}

We summarise the AGI$\to$ASI landscape so that the grounding reframing of
Section~\ref{sec:grounding-thesis} is recognisable as the \emph{same}
landscape, not a strawman.

\subsection{Four pathways}

The literature converges on four candidate pathways from AGI to
ASI~\citep{genewein2026agi,bostrom2014,morris2024levels}:
\begin{enumerate}
  \item \textbf{Scaling}: continued growth in compute, model size, and
        training data~\citep{Kaplan2020ScalingLaws,Hoffmann2022Chinchilla,
        epoch2024trainingcompute}. The bitter
        lesson~\citep{sutton2019bitter} is a claim about \emph{algorithms}
        (general methods leveraging compute beat clever human-designed
        priors); our position is about \emph{substrates} (which
        environment families a given algorithm trains and deploys on).
        The two axes compose: a bitter-lesson-compliant algorithm
        deployed against a high-$\delta$ substrate dominates both an
        algorithm with hand-designed priors and a bitter-lesson
        algorithm deployed against a low-$\delta$ substrate.
  \item \textbf{Paradigm shift}: replacement or augmentation of the
        present neural paradigm by new architectures, world models, or
        test-time search~\citep{gu2024mamba, Hafner2020Dreamer,
        bruce2024genie, snell2024scaling}.
  \item \textbf{Recursive self-improvement}: models that modify their
        own training process, data, or
        architecture~\citep{davidson2026does, lu2024ai,
        Real2020AutoMLZero, schmidhuber2003godel}, including the Darwin
        G\"{o}del Machine~\citep{zhang2025dgm}, the AI
        Scientist-v2~\citep{sakana2024aiscientistv2}, and empirical
        studies on safe self-improvement~\citep{anthropic2025rsi}.
  \item \textbf{Multi-agent coordination}: virtual agent economies in
        which competence is distributed across specialised
        agents~\citep{tomasev2025virtual, drexler2019reframing,
        park2023generative, wu2023autogen, hong2024metagpt,
        sumers2024coala}.
\end{enumerate}

\subsection{Six bottlenecks}

A parallel literature documents the frictions that may prevent these
pathways from succeeding~\citep{genewein2026agi}:
the \gls{datawall}~\citep{Villalobos2024WillRunOut, shumailov2024ai}; the
economic cost of frontier
training~\citep{agrawal2025economics, erdil2025gate}; paradigm
limits~\citep{chollet2019measure}; diminishing returns in
research~\citep{bloom2020ideas}; the \gls{abstractionbarrier} and
\gls{embodiedbottleneck}~\citep{lawrence2024atomic, ortega2021shaking};
and deliberate slowdown via
governance~\citep{anderljung2023frontier, eu2024aiact, bengio2024isr}.

\subsection{From quantitative scaling to grounded scaling}

We propose a single organising distinction:
\begin{itemize}
\item \emph{Quantitative scaling} asks how much marginal capability is
      bought per marginal unit of effective compute, holding the data and
      verifier distribution fixed.
\item \emph{Grounded scaling} asks how the marginal capability per unit
      compute changes when the underlying data and verifier distribution
      grow in \emph{verifiable real-world signal}.
\end{itemize}

\noindent The recent empirical record is consistent with the view that
quantitative scaling alone reaches a ceiling set by the second
quantity~\citep{Hoffmann2022Chinchilla, ho2025rosetta}. We take this as a
prompt to study the bottlenecks \emph{as grounding gaps}, which we do in
Section~\ref{sec:grounding-thesis}.

\subsection{The determinism gap in current agent environments}

A critical but under-studied dimension of the grounded-scaling problem is
\emph{environmental determinism}. Current agent benchmarks ---
WebArena~\citep{zhou2023webarena}, AgentBench~\citep{liu2024agentbench},
GAIA~\citep{mialon2024gaia}, OSWorld~\citep{xie2024osworld},
$\tau$-bench~\citep{yao2024taubench},
SWE-bench~\citep{jimenez2024swebench} --- are partially stochastic: web
pages change between runs, API responses vary, and session state
introduces irreproducibility. Real deployed environments are worse: A/B
testing, personalisation, rate limiters, and anti-bot mechanisms
deliberately introduce non-determinism optimised for human engagement.
For autonomous agents executing multi-step plans, this stochasticity is
catastrophic, and the published gap between single-turn LLM ability and
end-to-end task success is empirically
large~\citep{anthropic2024computeruse, openai2025operator}; agent-benchmark
reproducibility analyses confirm that environmental variance is a
first-order driver of this gap~\citep{kapoor2024agentsthatmatter}.
Closing this \emph{determinism gap} requires either redesigning environments
for agentic consumption or building agents robust enough to tolerate arbitrary
non-determinism. We argue in Section~\ref{sec:grounding-thesis} that the
former is more tractable and more leveraged.

Cognitive-architecture responses to this failure
profile~\citep{sumers2024coala, shinn2023reflexion, yao2023react,
wang2024voyager} add memory, planning, and skill-acquisition layers
\emph{around} the model. Our position is complementary: such
architectures can mitigate but not erase exponential degradation in
$\delta^k$, because every retry and replanning step itself consumes
a finite budget of deterministic queries (Section~\ref{sec:counter}).

\subsection{An upper bound: universal intelligence and the role of environment}

Theoretically, the continuum of machine intelligence is bounded above by
universal artificial intelligence formalised through Solomonoff prediction
and AIXI~\citep{Legg2007Universal, Hutter:04uaibook, Hutter:24uaibook2}.
Two corollaries follow. First, the score is environment-relative: a
learner universal-intelligent in distribution may still be undertrained on
the particular environment of interest. Second, every concrete agent
trains on some restricted family of environments, and the choice of that
family becomes a design variable. Our position is that which family one
privileges is the critical decision left under-discussed, and that
commercially self-sustaining supply environments satisfying
Definition~\ref{def:det-env} form a particularly informative family.

\subsection{Notation and assumed reader}

We assume familiarity with scaling-law notation, RLHF / RLAIF / RLVR
loops~\citep{christiano2018amplification, bai2022constitutional,
lambert2024rewardbench}, and multi-agent vocabulary. Glossary entries
recap key terms (\gls{grounding}, \gls{supplycertainty},
\gls{verifier}, \gls{c2m}, \gls{a2a}, \gls{dataflywheel},
\gls{groundedscaling}, \gls{modelcollapse}, \gls{deterministicenv}).
The Supply Certainty Index (SCI) and Determinism Maturity Model (DMM)
are defined where they first appear
(Sections~\ref{sec:five},~\ref{sec:dmm}).

\section{The Grounding Thesis and Two Formal Bounds}\label{sec:grounding-thesis}

The four bottlenecks of Section~\ref{sec:background} are specialisations
of a single deeper deficit: agents lack \emph{verifiable real-world
state} along four distinct modalities. We call the deficit
\gls{grounding} and the modalities \emph{data}, \emph{representational},
\emph{embodied}, and \emph{social-economic}. This section derives the
claim bottleneck by bottleneck, then states two formal bounds that make
the grounding framework measurable.

\subsection{Bottleneck-by-bottleneck derivation}

The data wall is conventionally framed as a token-exhaustion
problem~\citep{Villalobos2024WillRunOut}. We argue the binding constraint
is not token count but \emph{verifiable signal}. Self-generated text scales
indefinitely but loses anchoring to reality, with progressive distributional
narrowing~\citep{shumailov2024ai, gerstgrasser2024model}.
What is exhausted is not text; what is exhausted is
\emph{external arbitration that text is true}.
The abstraction barrier presents the same deficit in representational
form: a perfect imitator of human-labelled text inherits the human
ontology and, absent strong inductive bias, will not exceed
it~\citep{ortega2021shaking, chollet2019measure}. Crossing the barrier
requires evidence above linguistic supervision that a candidate concept
carves the world at a real joint --- exactly what interaction with real
distributions can supply and pure language modelling cannot.
The embodied bottleneck is analogous: recursive self-improvement is
upper-bounded by the slowest physical loop required to validate a
candidate improvement~\citep{lawrence2024atomic, genewein2026agi}. The
bound is not on compute or data, but on the latency and bandwidth of
\emph{physical verification}.
Finally, multi-agent trust: virtual economies cannot clear without
verifiable signals about identity, capability, inventory, price, and
outcome~\citep{tomasev2025virtual, drexler2019reframing}. Without such
signals every transaction degenerates into an information-asymmetric
``market for lemons''~\citep{akerlof1970market} or into a
hallucination-cascade among agents~\citep{ngo2022alignment}. The
bottleneck in each case is the supply of \emph{trust substrate}.

\subsection{Four modalities, one common deficit}

\begin{itemize}
  \item \textbf{Data grounding} = verifiable interaction signal.
  \item \textbf{Representational grounding} = evidence that a candidate
        concept is real.
  \item \textbf{Embodied grounding} = physical loops short enough to
        validate change.
  \item \textbf{Social-economic grounding} = trust signals that enable
        agent transactions.
\end{itemize}

\noindent Compute scaling amplifies each modality but supplies none. This
asymmetry is the core of our position: \emph{policies that buy more
grounding will dominate policies that buy more compute alone}, in regimes
where grounding is the binding constraint. The reframing is not a
redescription of well-known desiderata; it makes four things tractable
that were previously implicit: (i) the choice between simulation,
self-generation, and real environments becomes empirically arbitrable;
(ii) a single property --- privileged grounding substrate --- identifies
what any ``post-AGI data source'' must possess
(Section~\ref{sec:supply-privilege}); (iii) grounding (environment) is
separated from inductive bias (model), enabling proper attribution; and
(iv) two quantitative levers --- $\delta$ and $\varepsilon$ --- make the
debate measurable.

\subsection{The Determinism--Efficiency Bound}\label{sec:det-bound}

The intuition that environmental non-determinism degrades agent learning
can be formalised under explicit assumptions.

\begin{takeawaybox}{Box 2 $\vert$ Assumptions of Proposition~\ref{prop:det-bound}.}
\textbf{(A1) Per-step success probability $\delta$.} Determinism quality
$\delta(\mathcal{E}) \in [0,1]$ is the per-step success probability of
the environment against ground truth: the probability that, under a
declared intent-canonicalisation protocol that normalises a well-formed
query $q$, the environment's response yields an outcome the
independent verification channel of Definition~\ref{def:det-env}~(D3)
judges \emph{correct} against the relevant ground-truth state
(inventory, price, certification, settlement). Operationally, $\delta$
is estimated by sampling canonicalised queries and reporting the
empirical correct-fraction. The intent-canonicalisation protocol is a
prerequisite for cross-platform comparability and is itself a domain
artefact (see Remark~\ref{rem:cons-vs-corr}).

\textbf{(A2) Independent per-step verification.} Steps in a $k$-step
chain are verified independently by the environment; success of any
single step is a Bernoulli$(\delta)$ event with success defined as in
(A1).

\textbf{(A3) No retry budget.} The bound concerns first-attempt
chain success; retry-augmented agents are treated in
Remark~\ref{rem:retry} below.

\textbf{(A4) Bounded recovery cost.} Recovery from a failed step costs
strictly positive time; agents that ignore failed verification
verdicts are outside scope.

Assumption (A2) is the key simplification; correlated-step
chains are treated separately in Lemma~\ref{lem:correlated}.
\end{takeawaybox}

\begin{remark}[Consistency versus correctness]\label{rem:cons-vs-corr}
$\delta$ measures \emph{correctness} against an independent
ground-truth channel, not self-consistency under replay. An
environment that returns identical but incorrect responses to repeated
queries has high replay consistency yet $\delta = 0$; conversely, a
well-calibrated stochastic environment may have low replay consistency
but high $\delta$ on a per-query basis. Empirical $\delta$-measurement
protocols must therefore use the (D3) verification channel rather than
self-replication, and report the canonicalisation rule applied to the
intent space (otherwise inter-platform $\delta$ estimates are not
comparable).
\end{remark}

\begin{proposition}[Determinism--Efficiency Bound]\label{prop:det-bound}
Under assumptions (A1)--(A4), let $\mathcal{E}$ be an environment with
determinism quality $\delta \equiv \delta(\mathcal{E})$. For a learner
with sample budget $n$, the effective sample size is $n_{\mathrm{eff}}
= \delta \cdot n$. Consequently:
\begin{enumerate}
  \item For any target capability threshold $\theta$, the sample
        complexity satisfies
        $n(\theta, \mathcal{E}) = \Omega(1/\delta)$.
  \item Chain-task success probability for a $k$-step task under
        independent per-step verification degrades as
        $P_{\mathrm{success}}(k) \leq \delta^k$.
\end{enumerate}
\end{proposition}

\begin{proof}[Proof sketch]
For claim (1): under (A1) each sample is independently correct with
probability $\delta$ against ground truth, so contributes useful
gradient information with that probability; incorrect samples must be
filtered or discarded, yielding $\Omega(1/\delta)$ sample-complexity
overhead. For claim (2): under (A2), $k$ sequential steps are
independent Bernoulli$(\delta)$ trials, so the joint success
probability is $\delta^k$ exactly. The qualification ``$\leq$'' rather
than ``$=$'' covers correlated and adversarial step structures
(Lemma~\ref{lem:correlated}).
\end{proof}

\paragraph{Novelty and implications.}
The arithmetic $\delta^k$ is standard independent-Bernoulli
composition. What is novel is neither the math nor the qualitative
intuition (every reliability engineer knows chain failures compound).
What we propose is a \emph{measurement programme}: that $\delta$
should be instrumented, published, and tracked at the environment
level the way uptime is tracked at the service level, with the
specific operational consequences (the DMM, the SCI, the
investment-thesis frame) following from that instrumentation rather
than from any new theorem. Even moderate non-determinism ($\delta =
0.9$) produces catastrophic failure rates for long chains
($\delta^{10} \approx 0.35$). Investing in environment determinism has
\emph{exponential} returns for chain-task success, whereas model
robustness yields at best linear improvements. This asymmetry is the
core argument for prioritising environment redesign over model
hardening.

\paragraph{Relationship to SRE and reproducible-build traditions.}
$\delta$ differs from conventional service-level objectives (SLOs) in
site reliability engineering along three axes: (1)~SLOs measure
per-request \emph{availability} (uptime, latency percentiles);
$\delta$ measures \emph{semantic correctness} against ground truth.
A service may have 99.99\% uptime (SLO met) yet personalise responses
on every call ($\delta \ll 1$). (2)~SLOs are per-endpoint; $\delta$
compounds multiplicatively across a $k$-step chain, making the
exponential degradation visible only at the workflow level. (3)~SLOs
target human-tolerance thresholds; $\delta$ targets
algorithmic-consumption requirements where even small per-step noise
is catastrophic at chain length. $\delta$ is therefore a
\emph{chain-aware, semantic, agent-centric} reliability concept that
existing SRE frameworks do not capture. There is also a longer
software-engineering lineage worth acknowledging: the
reproducible-build tradition (hermetic builds, content-addressed
storage, deterministic compilation; cf.\ Nix~\citep{dolstra2004nix})
has spent two decades developing exactly the determinism vocabulary
the present paper imports for agent environments. Our contribution is
not to invent determinism as an engineering goal but to specialise it
--- from \emph{byte-equivalent build outputs} to \emph{semantically
stable agentic responses at chain length}.

\begin{lemma}[Correlation reshapes but does not erase exponential degradation]\label{lem:correlated}
Suppose the $k$-chain steps share latent session state $s$ such that
$\Pr[\mathrm{step\;}i \text{ succeeds} \mid s] = \delta_i(s)$ and,
\emph{conditional on $s$, step-success events are independent}. Let
$\bar\delta_i := \mathbb{E}_s[\delta_i(s)]$ denote the marginal
per-step determinism. Then chain-task success is
\[
P_{\mathrm{success}}(k) \;=\; \mathbb{E}_s\!\left[ \prod_{i=1}^{k}
\delta_i(s) \right].
\]
This quantity is reshaped by the correlation structure of
$\{\delta_i(s)\}_{i=1}^k$ across sessions:
\begin{enumerate}
  \item \emph{Negative correlation across steps} (compensating
        fluctuations) gives, via the FKG/Chebyshev sum inequality,
        $P_{\mathrm{success}}(k) \leq \prod_i \bar\delta_i$ --- a
        product strictly tighter than treating each marginal in
        isolation.
  \item \emph{Independence across sessions} gives
        $P_{\mathrm{success}}(k) = \prod_i \bar\delta_i$, which reduces
        to $\bar\delta^k$ only when the marginals are uniform.
  \item \emph{Positive correlation} (failures cluster by session ---
        a good session carries many steps) can yield
        $P_{\mathrm{success}}(k) > \prod_i \bar\delta_i$. In the limit,
        $P_{\mathrm{success}}(k) \to \Pr[\min_i \delta_i(s) = 1]$ as
        $k \to \infty$: the chain succeeds asymptotically only on the
        measure of deterministic-good sessions.
\end{enumerate}
In all three regimes, $P_{\mathrm{success}}(k) \to 0$ exponentially
in $k$ whenever $\Pr[\min_i \delta_i(s) = 1] = 0$ --- i.e.\ whenever
the environment admits no positive-measure deterministic
sub-environment. The position's qualitative claim --- chain-task
success collapses for any environment lacking a deterministic
sub-environment of positive measure --- is therefore robust to the
correlation structure of the steps. The escape clause
($\Pr[\min_i \delta_i(s) = 1] > 0$) is exactly the regime that $D3$
and $D4$ platforms (Section~\ref{sec:dmm}) aspire to engineer.
\end{lemma}

\begin{remark}[Retries do not erase the bound]\label{rem:retry}
A retry-augmented agent with $r$ retries per step achieves per-step
success $1-(1-\delta)^r$, yielding chain-success
$\left[1-(1-\delta)^r\right]^k$. The per-step retry budget $r$ is
bounded by latency, cost, and rate-limit constraints; for the
operator-class deployments cited in Section~\ref{sec:why-now} this
budget is small. When the retry budget is \emph{shared} across the
chain (a total budget $B$ distributed over $k$ steps), per-step $r$
scales as $B/k$ and the chain-success expression worsens with $k$
rather than merely tracking it: the exponential shape is reinforced
by realistic budget constraints. (Throughout the paper, $B$ denotes
the total retry budget and $r$ the per-step budget; we use $B$ rather
than $R$ to avoid clashing with the bounded reward $R$ of
Proposition~\ref{prop:goodhart-floor}.)
\end{remark}

\subsection{The Verifier--Goodharting Floor}\label{sec:goodhart-floor}

Proposition~\ref{prop:det-bound} treats verifier signals as ground
truth. In practice, every verifier is a learned or rule-based proxy with
irreducible error~\citep{casper2023open, gao2023scaling,
lambert2024rewardbench}; optimising hard against the proxy Goodharts the
true objective~\citep{goodhart1975measure}.

\begin{takeawaybox}{Box 3 $\vert$ Assumptions of Proposition~\ref{prop:goodhart-floor}.}
\textbf{(B1) Bounded verifier KL.} The verifier-induced outcome
distribution $V$ has bounded KL divergence from the true outcome
distribution $V^\star$: $D_{\mathrm{KL}}(V \,\|\, V^\star) \leq
\varepsilon$.

\textbf{(B2) Optimisation against $V$.} The training procedure
optimises a bounded reward $R$ under expectation in $V$ exactly, not
under expectation in $V^\star$.

\textbf{(B3) Bounded reward.} $R$ is bounded with
$\|R\|_\infty \leq C$.
\end{takeawaybox}

\begin{proposition}[Verifier--Goodharting Floor]\label{prop:goodhart-floor}
Under (B1)--(B3), any policy $\pi$ whose optimisation is shaped by
$\mathbb{E}_V[R]$ rather than $\mathbb{E}_{V^\star}[R]$ satisfies
\[
\bigl|\mathbb{E}_V[R] - \mathbb{E}_{V^\star}[R]\bigr| \;\leq\;
2C \cdot \mathrm{TV}(V,\,V^\star) \;\leq\; C\sqrt{2\varepsilon},
\]
where the first inequality is the standard bounded-function bound on
total variation and the second is Pinsker's inequality. Equivalently,
\[
\mathbb{E}_{V^\star}[R] \;\geq\; \mathbb{E}_V[R] \;-\; C\sqrt{2\varepsilon}.
\]
The bound is in general unimprovable using only (B1); tighter bounds
require additional structure on $R$ or on the verifier mismatch (e.g.\
Bretagnolle--Huber for two-point distinguishability, $f$-divergence
inequalities for smoother $V/V^\star$ relations).
\end{proposition}

\paragraph{Note on smoothness regime.} Earlier versions of this work
used a Lipschitz hypothesis on $R$ together with Pinsker's inequality;
this is a topological mismatch (Lipschitz-smoothness controls the
Wasserstein-1 / Kantorovich--Rubinstein gap, while Pinsker controls
total variation). The bounded-reward formulation above is the
correct pairing for the KL hypothesis (B1).

\noindent The flywheel asymptote of Section~\ref{sec:flywheel} is bounded
by this floor: the policy gain from one more order of magnitude of
verifier-gated training is dominated by $\varepsilon$, not by nominal
scale. This is the regime where adaptive, task-conditioned
verifiers~\citep{ding2026adarubric} and weak-to-strong
supervision~\citep{burns2024weaktostrong} pay back: not by making any
single verifier perfect, but by lowering $\varepsilon$ faster than the
model overfits the proxy.

\begin{table}[!b]
  \centering
  \small
  \begin{tabularx}{\textwidth}{p{3.0cm} p{3.2cm} p{2.6cm} X}
    \toprule
    \textbf{Bottleneck} &
    \textbf{Grounding modality} &
    \textbf{Supply property} &
    \textbf{Mechanism (privilege argument)} \\
    \midrule
    Data wall              & Data grounding              & Thick           & Verifiable interaction signal and transaction rewards expand coverage without model-collapse anchoring loss. \\
    Abstraction barrier    & Representational grounding  & Understandable  & Multimodal concept discovery over real distributions surfaces concepts beyond curated text labels. \\
    Embodied bottleneck    & Embodied grounding          & Customisable    & Customer-to-manufacturer loops shorten the physical-validation latency bounding recursive self-improvement. \\
    Multi-agent trust      & Social-economic grounding   & Trustworthy     & Decision-grade fields (inventory, qualifications, settlement) form the agent-to-agent trust substrate. \\
    Matching structure     & Social-economic grounding   & Comparable      & Same-item and sourcing networks structure the supply network into multi-agent scheduling. \\
    \bottomrule
  \end{tabularx}
  \caption{Bottleneck $\to$ grounding modality $\to$ supply-certainty
  property. The five rows form the conceptual backbone shared by the rest
  of the paper; Figure~\ref{fig:grounding-map} renders the mapping
  graphically and Section~\ref{sec:five} operationalises each property.}
  \label{tab:bottleneck-grounding}
\end{table}

\section{Privileged Grounding Substrates: Five Sufficiency Conditions}\label{sec:supply-privilege}

Of the many environment families one might propose to supply grounding ---
simulators, game worlds, scientific instruments, sandboxes for
self-play~\citep{Silver2017GoZero, bauer2023human}, web-scale interaction
loggers --- we ask: \emph{which families are structurally sufficient?}
Rather than arguing from a single domain, we identify five conditions
that jointly characterise a \emph{privileged grounding environment}.
Any concrete environment satisfying all five qualifies; failure on any
one disqualifies.

\begin{takeawaybox}{Box 4 $\vert$ Five sufficiency conditions for a privileged grounding environment.}
\textbf{(i) Economic self-sufficiency.} The environment carries its own
reward stream (transactions, fulfilment, returns, repeat engagement) and
its own compute budget (real user behaviour). It does not require subsidy
from the AGI programme to remain alive.

\textbf{(ii) Verifiability.} Outcomes are checkable against ground truth
through an independent channel --- settlement, physical delivery,
certification, inspection --- ideally suited to learned verifiers and
reward models.

\textbf{(iii) Interactivity.} The environment supports both a
high-frequency \emph{digital} loop and a low-latency, shortenable
\emph{physical} loop. Information and embodied grounding are co-present.

\textbf{(iv) Scale and multimodal richness.} The state space is massive,
long-tail, multimodal, and its ontology has been only partially exhausted
by human catalogs.

\textbf{(v) Deterministic interface guarantee.} The environment provides
stable, faithfully-ranked, bounded-latency responses suitable for
algorithmic consumption (Definition~\ref{def:det-env}).
\end{takeawaybox}

\paragraph{Illustrative domains.}
Environments satisfying (i)--(v) include, but are not limited to:
industrial B2B sourcing platforms, B2C retail with fulfilment
verification, pharmaceutical supply chains with batch traceability,
agricultural commodity exchanges with quality grading, automotive
supplier networks with manufacturability validation, and healthcare
procurement with credentialing. Financial exchanges satisfy (i), (ii),
(v) but lack multimodal richness (iv); scientific instruments satisfy
(ii) but rarely (i). The intersection is domain-diagnostic, not
domain-specific.

\paragraph{Honesty about evidentiary base.} We acknowledge that the
strongest concrete instantiation of the SCI / DMM machinery available
to us is sourcing-class commerce, and that the other domains above
are listed as \emph{predicted} qualifiers rather than \emph{measured}
ones. The worked SCI example in Section~\ref{sec:sci-example}
exercises three domains (B2B sourcing, financial settlement, B2C
retail) to demonstrate that the formal apparatus accommodates
domain-relevance masking, but a fully evidenced pharmaceutical or
healthcare instantiation is future work. The position is therefore
strongest as written for sourcing-class environments and progressively
weaker as one moves to domains where the verifier-channel infrastructure
is less mature; readers in those domains should treat the position as
a hypothesis to test rather than a result to apply.

\subsection{Economic self-sufficiency}

The dominant grounding alternative is large-scale simulation. Simulation
has succeeded spectacularly in narrow domains --- Go, chess, single-player
games~\citep{Silver2017GoZero, Schrittwieser2020MuZero} --- but its
scaling cost is borne by the AGI programme itself. An environment
satisfying condition (i) generates its reward stream as a by-product of
its own economic activity: buyers transact, sellers fulfil, disputes
settle --- all without annotation budget. The economic sustainability
literature has begun to note this
asymmetry~\citep{agrawal2025economics, erdil2025gate, acemoglu2023power,
brynjolfsson2022turing}, but its grounding implications have received
little attention.

\subsection{Verifiability}

Verifiers are the rate-limiting layer in
post-training~\citep{lambert2024rewardbench, bai2022constitutional,
ouyang2022instructgpt}. A verifier is only as useful as the ground truth
it can compare against. Environments satisfying (ii) provide ground
truths --- a payment cleared, a parcel delivered, a credential valid, a
manufactured part to-spec --- that are observable, adversarially robust at
population scale, and not subject to the preference-elicitation artefacts
of single-turn human feedback. We argue in Section~\ref{sec:flywheel}
that this property gives the data flywheel its robustness against
\gls{modelcollapse}~\citep{shumailov2024ai}, modulo
Proposition~\ref{prop:goodhart-floor}.

\subsection{Interactivity (digital and physical)}

Most digital environments support cheap, high-frequency interaction but no
physical loop; most physical environments support a physical loop but no
cheap digital loop. An environment satisfying (iii) supports both: a
digital loop (e.g.\ spot-matching, inquiry routing) and a physical loop
(e.g.\ \gls{c2m} manufacturing, clinical trial feedback). Where embodied
bottlenecks are real~\citep{lawrence2024atomic}, the physical loop's
\emph{latency} is the binding term; qualifying environments contain
explicit engineering levers to shorten it (manufacturability constraints,
capacity matching, rapid prototyping), which simulation does not.

\subsection{Scale and multimodal richness}

Web-scale text covers what people \emph{say} about things. Catalogs cover
what producers \emph{declare}. Neither covers what the things \emph{are}.
The full multimodal record of a single industrial category spans
engineering drawings, CAD geometry, materials data, process descriptions,
and end-user specifications --- a space the public web only thinly
samples. Any environment satisfying (iv) contains this kind of
unexhausted multimodal distribution, creating room for concept discovery
beyond the human ontology and thereby challenging the abstraction barrier.

\subsection{Deterministic interface guarantee}

Condition~(v) reflects Proposition~\ref{prop:det-bound}. An environment
may satisfy (i)--(iv) but still be hostile to agents if its interfaces
are designed for human browsing: shuffled search results,
session-dependent pricing, variable API latency. The deterministic
interface guarantee demands a programmatic access layer satisfying
(D1)--(D4). The distinction is concrete: a consumer-facing search
(personalised, diversified, attention-optimised) versus a
programmatic-first API (stable rankings, bounded latency, verifiable
state). The former is optimised for engagement; the latter for
algorithmic consumption.

\subsection{Sim-to-real, fairly considered}\label{sec:sim2real-engagement}

The most credible alternative is ``simulate, with domain randomisation,
until the model generalises''~\citep{tobin2017domain, peng2018simtoreal,
akkaya2019solving, andrychowicz2020learning, lee2020learning,
bruce2024genie, Hafner2020Dreamer}. This programme has won in several
domains, with high-evidence cases including OpenAI's dexterous in-hand
manipulation~\citep{andrychowicz2020learning, akkaya2019solving},
quadrupedal locomotion over challenging
terrain~\citep{lee2020learning}, and autonomous-driving sensor stacks
with photorealistic simulation. We do not dismiss this evidence; we
bound it.

Sim-to-real nonetheless stays outside the sufficient set even where it
succeeds locally, for three structural reasons. The first is an
ontology cap: simulation can only generate states its physics engine
and asset library anticipate, whereas real environments satisfying
(i)--(v) routinely produce states no designer foresaw (novel
compositions, unexpected failure modes, emergent demand patterns,
adversarial counterparty behaviour). The second is a verifier
ground-truth shift: a sim-to-real verifier uses the simulator's
physics as ground truth and only discovers the gap on costly real
trials, while an environment satisfying (ii) uses real-world outcomes
as ground truth from the first interaction --- a stronger
verifier-quality regime in Proposition~\ref{prop:goodhart-floor}'s
$\varepsilon$. The third concerns reward provenance: simulators force
reward to be specified, which is a known source of
hacking~\citep{casper2023open, gao2023scaling}, while qualifying
environments let reward be observed (bounded by Goodharting but not by
specification error). Simulation therefore belongs in the
three-way comparison of OQ1 as a complement, not a substitute.

\subsection{Why these five together matter}

Several alternative environments meet a subset of (i)--(v). Web crawl is
self-sustaining and rich but only weakly verifiable and not
deterministic~\citep{Soldaini2024Dolma};
scientific instruments are highly verifiable but not self-sustaining;
robotics is interactive but neither self-sustaining nor scalably
multimodal~\citep{bauer2023human}; financial exchanges are deterministic
and verifiable but not multimodally rich. The intersection is rare. We
do not claim it is unique to any single industry; we claim that
economically self-sustaining commercial environments --- spanning
sourcing, retail, pharma, agriculture, automotive --- are the most
readily available family that demonstrably satisfies all five at
industrial scale today. This is an empirical observation, not a
definitional restriction: any future environment meeting (i)--(v)
automatically enters the privileged set.

\section{Operationalising Supply Certainty: Five Properties and the SCI}\label{sec:five}

The position of Section~\ref{sec:supply-privilege} is empty without
an operationalisation. We decompose \gls{supplycertainty} into five
properties that are each (i) defined qualitatively in the language of
agentic consumers, (ii) admit a generic measurable proxy, and (iii) map
onto exactly one of the four grounding modalities (with one closing
property serving the multi-agent
structure). Table~\ref{tab:bottleneck-grounding} summarises the
mapping; this section unfolds it and aggregates the five into a
single platform-comparable index --- the \emph{Supply Certainty Index}
(SCI).

\begin{figure}[t]
  \centering
  \resizebox{\linewidth}{!}{
\begin{tikzpicture}[
    font=\small,
    node distance=0.55cm and 1.6cm,
    bn/.style={rectangle, rounded corners=3pt, draw=black!55, line width=0.6pt,
               fill=black!4, minimum width=3.6cm, minimum height=0.7cm,
               align=center, inner sep=4pt},
    sn/.style={rectangle, rounded corners=3pt, draw=black!55, line width=0.6pt,
               fill=black!8, minimum width=3.2cm, minimum height=0.7cm,
               align=center, inner sep=4pt},
    cluster/.style={draw=black!25, rounded corners=6pt, line width=0.4pt,
                    inner sep=10pt},
    edgelabel/.style={font=\footnotesize\itshape, fill=white,
                      inner xsep=3pt, inner ysep=2pt},
    >={Stealth[length=4pt]}
  ]

  \node[bn] (DW) {Data Wall};
  \node[bn, below=of DW] (AB) {Abstraction Barrier};
  \node[bn, below=of AB] (EB) {Embodied Bottleneck};
  \node[bn, below=of EB] (MA) {Multi-Agent Trust};

  \begin{scope}[on background layer]
    \node[cluster, fit=(DW)(AB)(EB)(MA),
          label={[font=\small\bfseries]above:AGI$\to$ASI bottlenecks}] (BC) {};
  \end{scope}

  \node[sn, right=5.2cm of DW] (K) {Thick (well-stocked)};
  \node[sn, below=of K]        (U) {Understandable};
  \node[sn, below=of U]        (M) {Customisable};
  \node[sn, below=of M]        (T) {Trustworthy};
  \node[sn, below=of T]        (C) {Comparable};

  \begin{scope}[on background layer]
    \node[cluster, fit=(K)(U)(M)(T)(C),
          label={[font=\small\bfseries]above:Supply certainty: five properties}] (SC) {};
  \end{scope}

  \draw[->, thick] (DW.east) -- node[edgelabel]{verifiable interaction signal} (K.west);
  \draw[->, thick] (AB.east) -- node[edgelabel]{multimodal concept discovery} (U.west);
  \draw[->, thick] (EB.east) -- node[edgelabel]{C2M physical loop}            (M.west);
  \draw[->, thick] (MA.east) -- node[edgelabel]{A2A trust substrate}          (T.west);

  \draw[->, thick, dashed] (C.west) -|
        node[edgelabel, pos=0.25]{supply network} (MA.south);

\end{tikzpicture}}
  \caption{Grounding mappings. Four AGI$\to$ASI bottlenecks (left) are each
  addressed by one supply-certainty property (right). The closing dashed
  edge captures that comparability structures the supply network back into
  a multi-agent matching substrate.}
  \label{fig:grounding-map}
\end{figure}
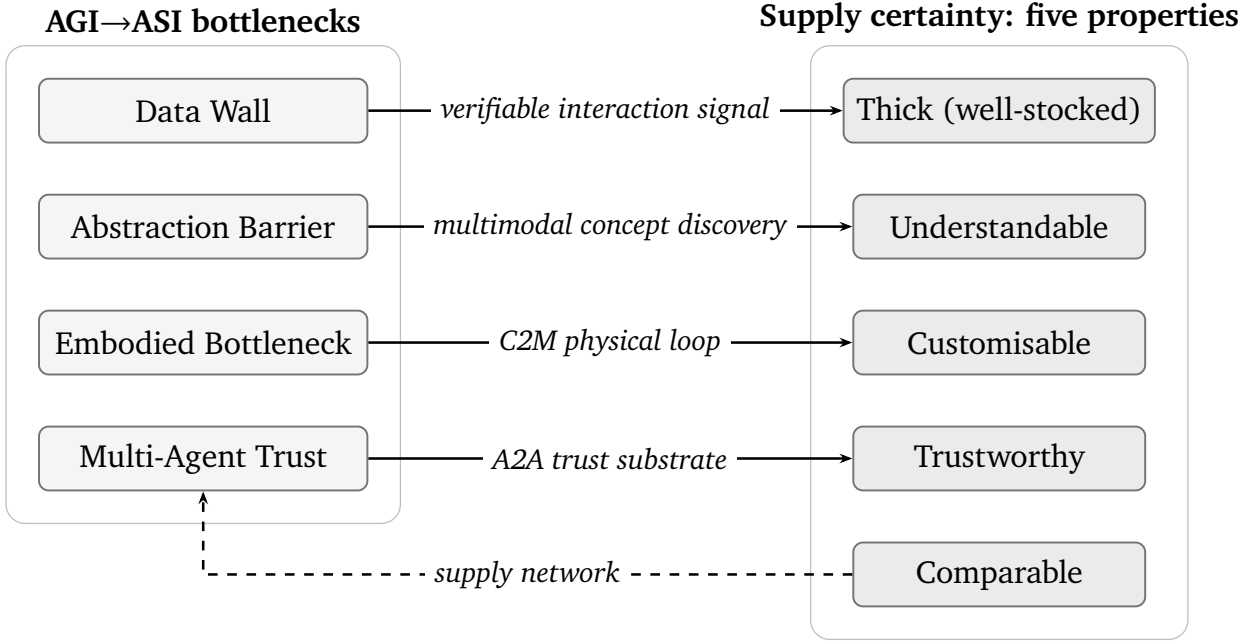

We describe the mappings in the order of
Figure~\ref{fig:grounding-map}.

\subsection{Thick (well-stocked) supply $\leftrightarrow$ data wall}

A supply environment is \emph{thick} for an agent if, conditional on a
demand intent, there exists with high probability a non-degenerate set of
supply candidates that satisfy the intent at non-degenerate granularity.
Thickness directly addresses the \gls{datawall}: every additional unit of
supply that the system can correctly cluster, describe, and match adds an
\emph{externally-arbitrated} interaction signal --- one that resists model
collapse because it is anchored in a downstream economic
event~\citep{shumailov2024ai, gerstgrasser2024model}. Measurable proxies
include coverage of long-tail categories, demand-to-supply translation
recall, and net new item rate per unit time.

\subsection{Understandable supply $\leftrightarrow$ abstraction barrier}

A supply environment is \emph{understandable} when an agent can read its
product ontology in the same vocabulary the environment is curated in: a
unified multimodal semantic base plus a structured
\gls{cpv}\,-style representation. Understandability is the operational
hook for the \gls{abstractionbarrier}. Crossing the barrier requires
discovering new conceptual primitives \emph{above} those in the curated
human ontology, possible only if the system can compare candidate
primitives against the unfiltered multimodal distribution of real
goods~\citep{ortega2021shaking, srivastava2022beyond}. Measurable proxies
include precision/recall of same-item clustering, attribute-governance
coverage, and verified knowledge graph size.

\subsection{Customisable supply $\leftrightarrow$ embodied bottleneck}

A supply environment is \emph{customisable} when a buyer-side
specification can be translated into a manufacturable bill of materials
and routed to a factory whose declared capabilities match.
Customisability is the engineering surface on which the
\gls{embodiedbottleneck} is fought. Each percentage point added to the
manufacturability rate and each minute removed from design-to-quote
latency shortens the embodied loop. Unlike purely scientific embodied
tasks, the physical environment here is already instrumented for
short-loop iteration. Measurable proxies include manufacturability rate
of generated artefacts, CAD kernel coverage, spec-to-quote latency
(inverted), and factory-capacity matching accuracy.

\subsection{Trustworthy supply $\leftrightarrow$ multi-agent trust}

A supply environment is \emph{trustworthy} for an agentic consumer when
the supply record exposes a complete set of \emph{decision-grade fields}:
inventory, invoicing, qualifications, certificates, price-and-freight,
weight-and-dimensions. Trustworthiness is the substrate of \gls{a2a}
interaction. A virtual agent economy~\citep{tomasev2025virtual} cannot
clear without verifier-checkable signals on price, capability, and
identity; in absence of such signals, the economy degenerates into
asymmetric-information failure~\citep{akerlof1970market} or hallucinated
trades. Measurable proxies include decision-grade-field coverage rate,
verified-supply set size, independent-judge precision/recall, and
field-record freshness.

\subsection{Comparable supply $\leftrightarrow$ supply network for multi-agent matching}

A supply environment is \emph{comparable} when supply items can be
clustered by demand-relevant equivalence (same-item) and ranked along
demand-relevant axes (price, lead time, service). Comparability is the
closing edge of Figure~\ref{fig:grounding-map}. Without it,
trustworthiness has nowhere to compose: an agent that can verify a single
seller cannot yet schedule a fleet. Comparability is the structural
prerequisite for multi-agent matching at scale. Measurable proxies
include cluster precision and Cov@k of same-item retrieval, and
independent-judge precision/recall.

\subsection{The Supply Certainty Index (SCI)}\label{sec:sci}

The five properties above are individually measurable but the literature
has lacked a single composite that platforms can be ranked on.
We propose the \emph{Supply Certainty Index}:

\begin{definition}[Supply Certainty Index, SCI]\label{def:sci}
For a supply environment $\mathcal{E}$ instrumented with the per-property
proxies above, the \emph{Supply Certainty Index}
$\mathrm{SCI}(\mathcal{E}) \in [0, 1]$ is the geometric mean of the
five property scores after each property's proxies are aggregated to a
$[0,1]$ score $S_p$:
\[
\mathrm{SCI}(\mathcal{E}) \;=\;
\left( S_{\mathrm{thick}} \cdot S_{\mathrm{und}} \cdot
       S_{\mathrm{cust}} \cdot S_{\mathrm{trust}} \cdot
       S_{\mathrm{cmp}} \right)^{1/5},
\qquad S_p \in [0,1] \;\forall p.
\]
The deterministic-interface condition of Definition~\ref{def:det-env}
contributes the \emph{gating multiplier} $\delta(\mathcal{E}) \in [0,1]$
yielding a deterministic-corrected score
\[
\mathrm{SCI}^{\delta}(\mathcal{E}) \;=\; \delta(\mathcal{E}) \cdot
\mathrm{SCI}(\mathcal{E}).
\]
\end{definition}

\paragraph{Why geometric mean.} A platform excellent on four properties
and zero on the fifth is not a privileged grounding substrate: the five
properties are non-substitutes (each addresses a distinct grounding
modality). The geometric mean enforces this: a single zero zeroes the
index. The $\delta$ multiplier further ensures that a platform with good
supply properties \emph{but a stochastic agent interface} is not
credited for what agents cannot reliably consume.

\paragraph{Why $[0,1]$ scores.} Absolute units vary by domain by orders
of magnitude. Each $S_p$ is the platform's measurement against a
domain-relative reference panel --- analogous to scaling-law
normalisation against compute budget.

\paragraph{Numerical convention for $S_p = 0$.} The geometric mean
is undefined in log-space when any $S_p = 0$ (its operational value
is $\mathrm{SCI} = 0$, but log-space computation gives $-\infty$).
For numerical implementations we adopt the convention that any score
below an $\varepsilon$-floor of $0.01$ is clamped to $\varepsilon =
0.01$ before aggregation, with the unclamped score reported
separately. This is a measurement convention, not a substantive
allowance: the operational meaning of $S_p \leq \varepsilon$ is
``platform fails this property in the relevant sense'' and the SCI
correctly reports a near-zero composite.

\paragraph{Domain relevance and N/A masking.} The geometric mean
penalises any zero score severely. This is intentional within domains
where all five properties are relevant. Where a property is
\emph{structurally absent} from a domain --- e.g.\ customisability for
a purely digital information marketplace with no physical fulfilment,
or trustworthiness in the sense of physical credentialing for a
pure financial-settlement environment --- the property is masked
rather than scored as zero:
\[
\mathrm{SCI}(\mathcal{E}; D) \;=\;
\left( \prod_{p \in P_D} S_p \right)^{1/|P_D|},
\qquad P_D \subseteq \{\mathrm{thick, und, cust, trust, cmp}\},
\]
where $P_D$ is the subset of \emph{applicable} properties for domain
$D$.

\paragraph{Pre-registration discipline (anti-manipulation).} To
prevent a platform or industry from gaming the index by declaring
inconvenient properties N/A, $P_D$ must be fixed by the reference
panel \emph{before} any platform measurement, published openly per
domain, and not amended on the basis of measurement outcomes (akin to
pre-registration in clinical trials). Platforms must report both the
masked SCI (over $P_D$) and the unmasked SCI (over all five
properties, with structurally-absent scores set to a stated
$\varepsilon$-floor) for transparency; reviewers and downstream
consumers can then compare. The N/A declaration is a property of the
\emph{domain} and \emph{reference panel maintainer}, not a property of
the \emph{platform under evaluation}.

\subsection{A worked SCI example}\label{sec:sci-example}

We illustrate the SCI on three hypothetical platforms in distinct
domains; numbers are stylised and meant to demonstrate the
construction, not to rank any real platform.

\begin{table}[!ht]
  \centering
  \footnotesize
  \setlength{\tabcolsep}{4pt}
  \begin{tabular}{@{}p{4.0cm} cccccc cc@{}}
    \toprule
    \textbf{Platform} & $S_{\mathrm{thk}}$ & $S_{\mathrm{und}}$ & $S_{\mathrm{cust}}$ & $S_{\mathrm{trust}}$ & $S_{\mathrm{cmp}}$ & $\delta$ & $\mathrm{SCI}$ & $\mathrm{SCI}^{\delta}$ \\
    \midrule
    P1: B2B sourcing ($D3$)
        & 0.85 & 0.70 & 0.60 & 0.80 & 0.75 & 0.92 & 0.73 & 0.67 \\
    P2: Financial settlement ($D3$)
        & 0.90 & 0.65 & N/A  & 0.88 & 0.80 & 0.95 & 0.80 & 0.76 \\
    P3: B2C retail ($D2$)
        & 0.95 & 0.55 & 0.30 & 0.40 & 0.60 & 0.72 & 0.52 & 0.37 \\
    \bottomrule
  \end{tabular}
  \caption{Illustrative SCI computation across three domains. SCI is
  the geometric mean over the applicable subset $P_D$; for P2, $P_D
  = \{$thick, understandable, trustworthy, comparable$\}$ excludes
  customisability under the financial-settlement reference panel.
  Worked arithmetic: P1 has
  $\mathrm{SCI} = (0.85\cdot0.70\cdot0.60\cdot0.80\cdot0.75)^{1/5}
  = 0.73$ and $\mathrm{SCI}^\delta = 0.92\cdot0.73 = 0.67$; P3 has
  $\mathrm{SCI} = 0.52$ and $\mathrm{SCI}^\delta = 0.37$. Pre-registered
  $P_D$ prevents post-hoc masking; see
  \S\ref{sec:sci}.}
  \label{tab:sci-worked}
\end{table}

\noindent Two observations: (i)~the $\delta$ gating multiplier
substantially separates $D3$ from $D2$ even at similar raw SCI ---
P3's SCI 0.52 drops to 0.37 after $\delta$-correction, while
P1's SCI 0.73 drops only to 0.67. (ii)~P2 scores highest on
$\mathrm{SCI}^{\delta}$ \emph{within its panel}, but its score is
\emph{not} comparable in absolute terms to P1 or P3 because the
domain panels differ. Cross-domain absolute comparison is out of
scope (Section~\ref{sec:conclusion} Limitations); ranking within a
domain (across platforms instantiating the same $P_D$) is the
intended operational use.

\subsection{Why a five-way decomposition?}

We could have proposed three properties or eight. The case for five
rests on Figure~\ref{fig:grounding-map}: four properties each carry one
grounding modality, and one closing property threads them into a
multi-agent structure. Whether the five collapse to fewer factors in
practice is an empirical question (OQ2 in
Section~\ref{sec:research-agenda}). We do not prescribe universal
thresholds; such thresholds are domain-dependent. The operational
question is whether relative improvements in each property translate
into measurable improvements in the corresponding grounding modality,
and whether the SCI composite predicts agent outcomes --- the empirical
content of OQ2.

\section{Verifiable Supply Data Flywheels as Post-AGI Scaling Resource}\label{sec:flywheel}

The previous section described a static property: supply certainty.
This section describes its dynamics. We claim that real, verifier-gated
interaction data forms a \gls{dataflywheel} that is \emph{sustainable}
in the post-data-wall regime, in a sense recursive self-generation is
not~\citep{shumailov2024ai, gerstgrasser2024model, dohmatob2024tale}.

\subsection{Evaluation as the loss function}

In deployed agentic systems the evaluation harness functions as the
training loss: bad cases, retries, low verifier scores, and tool
failures all back-propagate to upstream stages, so the data pipeline
behaves like a single forward pass with evaluation-shaped signal as
its gradient.

\subsection{Dual flywheel architecture}

Two loops coexist (Figure~\ref{fig:flywheel}):

\begin{figure}[t]
  \centering
  \resizebox{0.92\linewidth}{!}{
\begin{tikzpicture}[
    font=\small,
    node distance=1.0cm and 1.5cm,
    hub/.style={circle, draw=black!60, fill=black!8, line width=0.6pt,
                minimum size=2.0cm, align=center, inner sep=2pt,
                font=\small\bfseries},
    nd/.style={rectangle, rounded corners=3pt, draw=black!55, line width=0.5pt,
               fill=white, align=center, minimum height=0.7cm,
               minimum width=2.4cm, inner sep=4pt},
    >={Stealth[length=4pt]}
  ]

  \node[hub] (E) {Evaluation\\$\equiv$ loss};

  \node[nd, above=1.0cm of E]                 (J)   {Verifier /\\Judge};
  \node[nd, left=0.8cm of J]                  (BC)  {Deployment\\bad cases};
  \node[nd, right=0.8cm of J]                 (SFT) {SFT \& RL\\samples};
  \node[nd, right=2.0cm of SFT]               (BAS) {Foundation\\refresh};

  \draw[->, thick] (BC.east) -- (J.west);
  \draw[->, thick] (J.east)  -- (SFT.west);
  \draw[->, thick] (SFT.east) -- (BAS.west);

  \draw[->, thick] (BAS.south) |- (E.30);
  \draw[->, thick] (E.150)     -| (BC.south);

  \node[nd, below=1.0cm of E]                 (SK) {Skill\\routing};
  \node[nd, left=0.8cm of SK]                 (CA) {Agent call /\\task};
  \node[nd, right=0.8cm of SK]                (EX) {Tool /\\execution};
  \node[nd, right=2.0cm of EX]                (VF) {Outcome\\verifier};

  \draw[->, thick] (CA.east) -- (SK.west);
  \draw[->, thick] (SK.east) -- (EX.west);
  \draw[->, thick] (EX.east) -- (VF.west);

  \draw[->, thick] (VF.north) |- (E.330);
  \draw[->, thick] (E.210)    -| (CA.north);

  \node[font=\footnotesize\itshape, above=0.18 of BAS]
       {data flywheel (iteration layer)};
  \node[font=\footnotesize\itshape, below=0.18 of VF]
       {A2A flywheel (execution layer)};

  \path (VF.east)  ++(1.6, 0) coordinate (Tbot);
  \path (BAS.east) ++(1.6, 0) coordinate (Ttop);
  \draw[->, thick, dashed] (VF.east) -- (Tbot) -- (Ttop) -- (BAS.east);

  \node[font=\footnotesize\itshape, text width=3.0cm, align=left, anchor=west,
        inner xsep=2pt, inner ysep=0pt]
       at ($(Tbot)!0.5!(Ttop) + (0.20, 0)$)
       {Business-feedback trunk: real distribution, conversion, dispute, return $\Rightarrow$ training signal};

\end{tikzpicture}}
  \caption{Dual flywheel architecture. The data flywheel (top loop) turns
  deployment bad cases into refreshed foundation parameters. The A2A
  flywheel (bottom loop) turns agent calls into verifier-checked outcomes.
  Both are routed through a single ``evaluation as loss'' hub.}
  \label{fig:flywheel}
\end{figure}
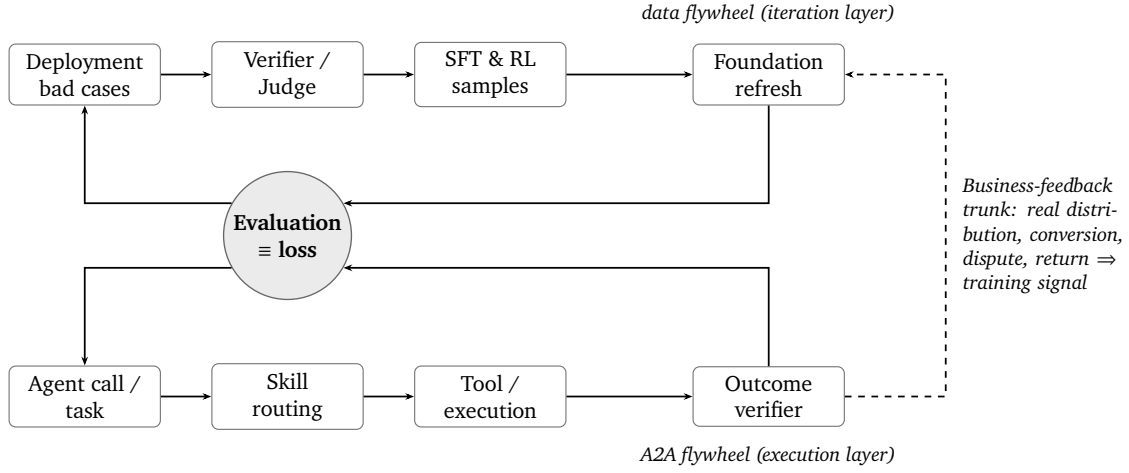

\begin{itemize}
  \item \emph{Data flywheel} (iteration layer): deployment bad cases
        $\to$ verifier/judge $\to$ SFT and RL training samples $\to$
        foundation refresh.
  \item \emph{A2A flywheel} (execution layer): agent call $\to$ skill
        routing $\to$ tool execution $\to$ outcome verifier.
\end{itemize}

\noindent The two loops share the evaluation hub; the
\emph{business-feedback trunk} connecting outcomes to foundation refresh
converts a one-way deployment system into a closed loop.

\subsection{Self-evolving environments: the complementary scaling axis}\label{sec:env-evolution}

A complementary direction has emerged in the self-evolving-agents
literature~\citep{gao2025selfevolvingsurvey, wang2024voyager,
shinn2023reflexion}: instead of (or in addition to) adapting the model
to a fixed environment, the \emph{environment itself} --- tools, skills,
memory, runtime context --- adapts to deployed trajectories.

\paragraph{The agentic-scaling trio.}
Three orthogonal scaling axes:
\begin{enumerate}
  \item \emph{Model adapts to environment}: SFT, RLHF, and RLVR refresh
        foundation weights~\citep{ouyang2022instructgpt,
        bai2022constitutional, wang2024executable}.
  \item \emph{Environment adapts to model}: skill libraries, tool graphs,
        and persistent context evolve from execution
        trajectories~\citep{yang2025skillopt, feng2026searl, zhang2025ace,
        wang2024voyager, sumers2024coala}.
  \item \emph{Co-evolution}: model and environment update jointly under a
        shared objective~\citep{wang2024sage, li2026arise,
        fang2025webevolver, zhang2025dgm}.
\end{enumerate}

\subsection{The Grounded Self-Evolution Convergence Condition}\label{sec:gsec}

We synthesise the mechanisms above into a formal condition connecting
environment-side self-evolution to the determinism framework.

\begin{takeawaybox}{Box 5 $\vert$ Grounded Self-Evolution Principle (qualitative).}
Environment-side evolution (skills, tools, memory, routing tables)
accumulates capability \emph{monotonically} if and only if the
verification signal exceeds a determinism threshold $\delta >
\delta_{\min}$ and a verifier-quality threshold $\varepsilon <
\varepsilon_{\max}$. Below either threshold, drift dominates
improvement.
\end{takeawaybox}

\begin{proposition}[Grounded Self-Evolution Convergence]\label{prop:gsec-formal}
Let $\Phi_t \in [0,1]$ denote a scalar quality measure of an
environment-side skill library at evolution step $t$, with bounded
per-step update $\Delta_t := \Phi_{t+1} - \Phi_t$ satisfying
$|\Delta_t| \leq c$ for $c > 0$. Suppose candidate updates are drawn
from a generation distribution $\mathcal{G}$ and gated by a verifier
$V$ whose behaviour is characterised by:
\begin{itemize}
  \item Reliability (replay consistency): $\delta \in [0,1]$,
        the probability that $V$ returns the same verdict on two
        i.i.d.\ trials of the same candidate.
  \item Validity (true-positive rate against ground truth):
        $1-\varepsilon$, the probability that $V$ accepts a candidate
        that is genuinely an improvement under $\mathcal{G}$.
\end{itemize}
(The $\delta$ here is the verifier's reliability axis, not the
environment-level correctness $\delta$ of Proposition~\ref{prop:det-bound};
the two coincide when the verifier under (D3) is the environment's
own ground-truth channel.)
Let
\[
  \mu \;:=\; \mathbb{E}_{\mathcal{G}}\!\bigl[\Delta_t \bigm| \text{accept; truly an improvement}\bigr] \;>\; 0,
  \qquad
  d \;:=\; \mathbb{E}_{\mathcal{G}}\!\bigl[-\Delta_t \bigm| \text{accept; false positive}\bigr] \;\geq\; 0,
\]
both conditional expectations under $\mathcal{G}$. Then the expected
per-step drift satisfies
\[
\mathbb{E}[\Delta_t] \;\geq\; \delta (1-\varepsilon) \mu \;-\; \delta
\varepsilon d \;-\; (1-\delta) \cdot c,
\]
and a sufficient condition for monotone improvement
($\mathbb{E}[\Delta_t] > 0$) is
\[
\delta \;>\; \frac{c}{c + (1-\varepsilon)\mu - \varepsilon d}
\qquad\text{whenever}\qquad
(1-\varepsilon)\mu > \varepsilon d.
\]
\end{proposition}

\begin{proof}[Proof sketch]
The verifier's behaviour on each candidate decomposes
non-orthogonally along two axes: \emph{reliability} (does $V$ return
a consistent verdict on replay?) and \emph{validity} (does the
verdict match ground truth on a genuine improvement?). Conditioning
on the reliability axis, with probability $\delta$ the verifier is
consistent --- in which case the validity axis yields either a
correct-accept (probability $1-\varepsilon$, contribution $+\mu$) or
a false-positive accept (probability $\varepsilon$, contribution
$-d$). With probability $1-\delta$ the verifier returns an
inconsistent verdict, in which case the update is treated as adding
worst-case noise of magnitude up to $c$. Aggregating gives the
displayed lower bound after replacing $\mathbb{E}[|\Delta_t|]$ by its
worst-case upper bound $c$ (a deliberate looseness; tighter bounds
are possible if $\mathcal{G}$ is constrained). Solving the lower bound
for positivity yields the sufficient condition.
\end{proof}

\paragraph{Note on tightness.}
The sufficient condition is loose by construction: replacing
$\mathbb{E}[|\Delta_t|]$ with the worst-case $c$ in the third term
overstates the noise contribution. Empirical operating regimes
(Section~\ref{sec:research-agenda}, OQ5) should therefore find
positive expected drift at $\delta$ values \emph{below} the
sufficient-condition threshold. The threshold is a conservative
guarantee, not a tight characterisation.

\paragraph{Implication.}
$\delta_{\min}$ and $\varepsilon_{\max}$ are the thresholds at which the
lower bound crosses zero. A skill library updated under sub-threshold
conditions random-walks; above-threshold, it accumulates. This unifies
findings from~\citet{yang2025skillopt} (1--4 accepted edits move
performance),~\citet{wang2024voyager} (skills improve only under
deterministic game logic), and~\citet{feng2026searl} (tool-graph quality
requires dense reward).

\subsection{Three mechanistic categories of environment-side evolution}

The diverse mechanisms unify into three categories:
\begin{itemize}
  \item \emph{Trajectory recycling}: verifier-checked trajectories are
        reused as supervision. SAGE trains skill creation via
        skill-augmented GRPO~\citep{wang2024sage}; ARISE evolves
        intrinsic skills under hierarchical RL~\citep{li2026arise};
        WebEvolver co-evolves a world model with a web
        agent~\citep{fang2025webevolver}; Voyager curates a skill
        library from interaction traces~\citep{wang2024voyager};
        Reflexion uses verbal reinforcement to recycle failed
        attempts~\citep{shinn2023reflexion}. Even \emph{failed}
        trajectories carry signal: AgentHER~\citep{ding2026agenther}
        demonstrates that hindsight relabeling of unsuccessful attempts
        produces high-quality training data, turning every failure into
        a training asset.
  \item \emph{Skill-as-text optimisation}: skills are represented as
        editable text artefacts optimised by a meta-learner.
        SkillOpt~\citep{yang2025skillopt} shows 1--4 accepted edits
        produce transferable improvements; recursive summarisation of
        execution history~\citep{wang2023recursively} compresses
        long-horizon experience into reusable skill primitives.
  \item \emph{Tool-graph consolidation}: execution traces are compiled
        into typed dependency graphs. SEARL~\citep{feng2026searl}
        jointly optimises policy and graph under dense reward. The
        cognitive-architectures view of~\citet{sumers2024coala} situates
        this inside a broader memory--planning--action framework.
\end{itemize}

\noindent In each case, the trajectory is the gradient --- the same logic
our \emph{evaluation as loss} principle articulates, generalised beyond
model weights to the entire agent stack.

\paragraph{Connection to determinism.}
When the environment satisfies (D1)--(D4), the verification signal used
to gate skill updates is reliable, and
Proposition~\ref{prop:gsec-formal}'s precondition is met. This is the
formal bridge between the ``deterministic agentic AI'' framing and the
self-evolving-agents literature. A privileged grounding environment is
valuable not only because real-supply trajectories train better models,
but because environment-side components --- skill libraries,
decision-grade fields, A2A routing tables --- can themselves accumulate
those trajectories. The data flywheel is the model-side realisation; the
skill flywheel is its environment-side counterpart. Recursive
self-improvement at the environment layer inherits the safety concerns of
model-side recursive self-improvement~\citep{davidson2026does,
anthropic2025rsi, zhang2025dgm}; the Verifier--Goodharting Floor applies
equally, since a skill library updating against a high-$\varepsilon$
verifier will Goodhart in the same way a model trained on its own outputs
would.

\subsection{Why this resists collapse, and where it relocates collapse}

The model-collapse literature is split:~\citet{shumailov2024ai} and
\citet{dohmatob2024tale} show that purely self-generated data streams
degrade representation diversity, while~\citet{gerstgrasser2024model}
demonstrate that \emph{accumulation} of real with synthetic data
avoids collapse. Our flywheel claim aligns with the latter: the SFT
distribution is anchored not in the model's generative distribution
but in \emph{verified-true} cases drawn from real-world outcome
signal --- an instance of verifier-gated distillation, related in
spirit to RLAIF~\citep{bai2022constitutional, lee2023rlaif} and RLVR
with task-specific verifiers~\citep{kim2024prometheus,
wang2024executable, dubey2024llama3}.

This collapse-resistance is not unbounded. By
Proposition~\ref{prop:goodhart-floor}, a verifier with KL gap
$\varepsilon$ from ground truth lower-bounds the asymptotic policy
reward by $\mathbb{E}_V[R] - C\sqrt{2\varepsilon}$ for a $C$-bounded
reward. Verifier-gated distillation therefore does not eliminate
collapse so much as \emph{relocate} it: in the long-horizon limit the
SFT distribution converges towards the verifier's mode, which is
bounded away from ground truth by $\varepsilon$ in KL. The position
trades the Shumailov attractor (self-imitation) for the Goodhart
attractor (verifier-imitation). The advantage is that the Goodhart
attractor is bounded by an \emph{external} quantity ($\varepsilon$,
empirically measurable on a held-out adversarial panel) rather than
the model's internal generative entropy.

\subsection{Stability conditions and verifier-quality ceiling}

The flywheel is only as stable as its verifier, sample selection, and
catastrophic-forgetting defence~\citep{kirkpatrick2017overcoming,
dohmatob2024tale}. Recent work addresses each: instruction
backtranslation~\citep{li2023selfalignment}, weak-to-strong
generalisation~\citep{burns2024weaktostrong}, and judge-quality
benchmarks~\citep{lambert2024rewardbench}. They collectively pin down
the operating envelope inside which the flywheel accumulates signal.

One path to lowering $\varepsilon$ in practice is task-adaptive
verifiers that generate assessment criteria conditioned on the specific
workflow~\citep{ding2026adarubric}; such approaches directly reduce
out-of-distribution error and raise the flywheel's quality ceiling.
Where this matters empirically is OQ5 of
Section~\ref{sec:research-agenda}.

\section{Supply Certainty as Trust Substrate for Agent Economies}\label{sec:agent-econ}

The fourth grounding modality is social-economic. We treat it last because
its argument depends on the operationalisations of
Sections~\ref{sec:five}--\ref{sec:flywheel}.

\subsection{The market-clearing problem for virtual agent economies}

Several recent position papers anticipate \emph{virtual agent economies}
in which discovery, negotiation, and settlement are conducted between
agents~\citep{tomasev2025virtual, drexler2019reframing}. The
infrastructure conversation has focused on protocol design: tool-calling
APIs~\citep{schick2023toolformer}, model-context
protocols~\citep{anthropic2024mcp}, agent-to-agent
schedulers~\citep{google2025a2a}, and payment-side agent
toolkits~\citep{stripe2025agenttoolkit}. Less discussed is the
\emph{informational} pre-condition for any such market to clear.

\paragraph{Claim.}
A virtual agent economy cannot clear without verifiable supply-side
signals about identity, capability, inventory, price, and outcome. Where
such signals are missing, the economy degenerates either into Akerlof
asymmetric-information failure~\citep{akerlof1970market} or into a
hallucination cascade~\citep{ngo2022alignment, weidinger2022taxonomy}.

\subsection{Multi-agent scaling laws, conditioned on supply trust and determinism}

The standard multi-agent scaling discussion treats matching quality as a
function of agent population and interaction
density~\citep{leibo2018malthusian, liu2025spiral, panait2005cooperative,
park2023generative, wu2023autogen, hong2024metagpt}. We propose that, in
the agent-economy regime, this scaling is gated by \emph{supply trust}
and \emph{environment determinism}:

\[
  Q_{\mathrm{match}} \;\approx\; f(N_{\mathrm{agents}}, \rho_{\mathrm{interaction}},
                                  \;\tau_{\mathrm{supply}}, \;\delta_{\mathrm{env}}),
\]

where $\tau_{\mathrm{supply}}$ is the decision-grade-field coverage rate,
$\delta_{\mathrm{env}}$ is the determinism quality of
Definition~\ref{def:det-env}, and the SCI of Section~\ref{sec:sci} is a
composite across the five properties. Whether $\tau$ and $\delta$ enter
$f$ as multipliers, thresholds, or interact non-linearly is OQ4
(Section~\ref{sec:research-agenda}).

\subsection{Failure modes of insufficient trust}

The failure-mode taxonomy below parallels and specialises the
miscoordination / conflict / collusion taxonomy
of~\citet{hammond2025multiagent} to the supply-trust substrate. In
the under-trusted regime, three failure modes recur:
\begin{enumerate}
  \item \emph{Settlement abortion.} An agent cannot commit because the
        price, inventory, or qualification record is stale or absent.
  \item \emph{Epistemic hijacking.} An adversarial supplier exploits the
        asymmetry between linguistic richness and verifier
        granularity~\citep{zou2023universal, wei2024jailbroken}.
  \item \emph{Hallucination cascade.} An agent fabricates a supply
        attribute; downstream agents trust it; the fabrication enters the
        data flywheel.
\end{enumerate}

\noindent Each is a direct consequence of a missing decision-grade field;
each is tractable as a coverage problem on the trustworthy-supply property.

\subsection{Failure modes of insufficient determinism}

Non-deterministic environments introduce three additional multi-agent
failure modes:
\begin{enumerate}
  \item \emph{Rerouting cascades.} Non-deterministic responses force
        rerouting; if many agents reroute simultaneously, the alternative
        platform's determinism degrades --- a cascade failure.
  \item \emph{Coordination impossibility.} Multi-agent scheduling
        requires consistent supply-state views. Personalisation or A/B
        testing returns different answers to different agents, fragmenting
        coordination.
  \item \emph{Strategy instability.} Algorithmic procurement strategies
        calibrated on historical data become unreliable when the
        environment's response distribution shifts non-stationarily.
\end{enumerate}

\subsection{Skillification as the trust interface}

The skill-registry abstraction (a typed, SLA-monitored,
trace-instrumented inventory of callable capabilities) is the right
interface contract for agent-economy trust. It exposes verifier-checkable
promises about each call and bounds the cost of hallucinated supply
assertions~\citep{anthropic2024mcp, google2025a2a}. The mechanism by
which such a registry evolves --- trajectory recycling, skill-as-text
optimisation, tool-graph consolidation --- is treated in
Section~\ref{sec:env-evolution}; here we note only that grounding
quality bounds both foundation-weight evolution and skill-registry
evolution (Proposition~\ref{prop:gsec-formal}). The interface
contract and the evolution mechanism are complementary: the registry
\emph{is} the operational artefact that the §\ref{sec:env-evolution}
skill flywheel updates.

\subsection{Investment theses for the determinism economy}\label{sec:invest-theses}

If the position is correct, capital should concentrate on primitives that
produce, package, and price deterministic verifiable signal. We name five
categories the position predicts will dominate, in keeping with recent
``agent economy'' framing~\citep{a16z2026agenteconomy,
karpathy2025software3}.

\begin{takeawaybox}{Box 6 $\vert$ Five investment-thesis categories implied by the position.}
\textbf{(I1) Verifier-as-a-Service.} Adaptive,
task-conditioned verifiers~\citep{ding2026adarubric, lambert2024rewardbench}
that lower Proposition~\ref{prop:goodhart-floor}'s floor.

\textbf{(I2) Determinism Infrastructure.} Programmatic-first APIs over
legacy systems satisfying (D1)--(D4).

\textbf{(I3) Agent Commerce Protocols.} MCP / A2A /
agent-toolkit~\citep{anthropic2024mcp, google2025a2a,
stripe2025agenttoolkit, openai2024function} extended with verifiable
trust, dispute resolution, and payment-finality guarantees.

\textbf{(I4) Skill Registry Marketplaces.} Typed, SLA-monitored
marketplaces for callable agent
skills~\citep{yang2025skillopt, feng2026searl} turning the skill
flywheel into a tradable asset class.

\textbf{(I5) Verifiable-Environment Benchmarks.} Public benchmarks
where the score is actual settlement, fulfilment, or manufacturability
--- extending~\citet{jimenez2024swebench, xie2024osworld,
yao2024taubench, min2025ecombench, bai2026industrybench,
qi2026industrybenchmipu} to the multi-domain SCI evaluation panel of
Section~\ref{sec:research-agenda}.
\end{takeawaybox}

\paragraph{What each thesis predicts.} (I1) verifier-quality improvement
will outpace model-quality improvement; (I2) the markup of deterministic
APIs over stochastic ones will widen with agent adoption; (I3) protocols
shipping verifier hooks will displace message-passing-only ones;
(I4) skill-registry marketplaces will become a separate category from SaaS;
(I5) benchmark-design competence will become a distinct investable. Each
prediction is independently refutable. We emphasise that each thesis
identifies a structurally privileged category, not a winner inside it; we
endorse no specific platform and have no financial relationship with any
referenced company.

\paragraph{Exclusive predictions: discriminating between positions.}
A reviewer may object that I1--I5 are categories any ``AI
infrastructure'' frame would predict regardless of the grounding
thesis. To make the position's predictions \emph{exclusive}, we pair
each thesis with a discriminator against the strongest competing
positions of Section~\ref{sec:counter}:
\begin{itemize}
  \item \emph{vs alignment-suffices} (\S\ref{sec:cp-rlhf}): alignment
        wins predicts I1 (verifier-as-a-Service) alone dominates;
        \emph{this position predicts I1 and I2 markups co-move},
        because better verifiers and deterministic infrastructure are
        complements not substitutes.
  \item \emph{vs sim-to-real} (\S\ref{sec:cp-sim2real}): sim-to-real
        predicts I5 dominates (simulator-based benchmarks suffice);
        \emph{this position predicts I5 settlement-grounded benchmarks
        outperform sim-only benchmarks on rank-correlation with
        production agentic success}.
  \item \emph{vs orchestration-first} (\S\ref{sec:cp-rlhf} variant):
        orchestration wins predicts I3 (agent commerce protocols
        with retry / message-passing) without I2 (determinism infra);
        \emph{this position predicts I3 protocols shipping verifier
        hooks displace message-passing-only protocols, and that the
        market gap appears at the $D2 \to D3$ transition}.
  \item \emph{vs AI-as-normal-technology} (\S\ref{sec:cp-normal}):
        normal-tech predicts I1--I5 are sociotechnical infrastructure
        plays with no sharp inflection; \emph{this position predicts a
        non-linear inflection at $D2 \to D3$ where the marginal
        category value steps rather than tracks}.
\end{itemize}
These are mutually discriminating; the position loses if the observed
investment-market trajectories match a competing position's predictions
better than ours.

\subsection{The Determinism Maturity Model}\label{sec:dmm}

Sections~\ref{sec:five}--\ref{sec:agent-econ} state what
deterministic supply environments do for grounding. This subsection states
what platform engineers should build: a five-level adoption ladder paired
with a reference architecture.

\begin{takeawaybox}{Box 7 $\vert$ Determinism Maturity Model (DMM) levels $D0$--$D4$.}
\textbf{$D0$ --- Human-only UI.} Programmatic access is blocked or
CAPTCHA-gated. $\delta$ unmeasured.

\textbf{$D1$ --- Minimal API surface.} A programmatic surface exists but
inherits human-UI semantics: shuffled rankings, session-conditioned
pricing, no SLA. $\delta$ typically $<0.5$ on long chains.

\textbf{$D2$ --- SLA-bounded API.} Bounded latency, versioned schemas,
uptime SLAs. Rankings may still be personalised; verifier channels
absent. $\delta \in [0.6, 0.8]$. The current state of major B2B APIs.

\textbf{$D3$ --- Stable rankings and verifier channel.} Rankings
deterministic for an intent/persona hash within a staleness bound;
verifier channel exposes inventory, qualification, and settlement-finality
assertions. $\delta > 0.9$ achievable. Crosses
Proposition~\ref{prop:gsec-formal}'s threshold.

\textbf{$D4$ --- Full deterministic agent interface.} Faithful ranking
(no exploration injection under agent persona); production skill-registry
contract with per-skill SLA, trace, and verifier telemetry. Published
$\mathrm{SCI}^{\delta}$ per agent persona. The platform's supply becomes
a first-class grounding substrate.
\end{takeawaybox}

\paragraph{Reading the ladder.} Levels are cumulative. The marginal cost
of climbing one level is dominated by the marginal benefit only above
Proposition~\ref{prop:gsec-formal}'s threshold --- typically at
$D2 \to D3$, where the verifier channel unlocks the data flywheel.
The position predicts that $D2$ platforms will face increasing pressure
to climb to $D3$ as \emph{agentic traffic} --- automated API calls
executing end-to-end workflows of three or more steps without human
gating --- grows from a fringe share of total traffic to a substantial
one. The exact tipping point is platform-specific; the falsifiable
claim is that the transition is non-linear in $\delta$ and that
platforms which stall at $D2$ will see disproportionate chain-task
degradation on their own agentic traffic
(Proposition~\ref{prop:det-bound}). The DMM is the substrate-side
complement of the agent-infrastructure agenda
of~\citet{chan2025infrastructure}: their identity, audit-trail, and
credentialing infrastructure describes what agents need to bring
\emph{to} the environment; the DMM describes what the environment
must expose \emph{to} agents. The two are complementary.

\paragraph{Reference architecture.}\label{sec:ref-arch}
A $D3$/$D4$ platform requires three components beyond a conventional API
gateway:

\subparagraph{Component A --- Verifier service.}
For each platform assertion (inventory, certification, settlement state):
(a) the assertion's value at query time; (b) source channel and
freshness; (c) an independent verification endpoint returning verdict
plus confidence; (d) verdict-history telemetry consumable by downstream
agents and the data flywheel. This component contains
Proposition~\ref{prop:goodhart-floor}'s $\varepsilon$ and should be
benchmarked against judge-quality
suites~\citep{lambert2024rewardbench, ding2026adarubric}.

\subparagraph{Component B --- Skill registry and A2A router.}
A typed inventory of callable capabilities with schema, SLA, trace
instrumentation, and verifier-outcome binding. The operational
counterpart of the skill flywheel (Section~\ref{sec:env-evolution}).

\subparagraph{Component C --- Determinism telemetry plane.}
A measurement plane publishing empirical $\delta$, $\varepsilon$, and
$\mathrm{SCI}^{\delta}$ per agent persona and per skill ---
contractually observable to platform consumers and closing the
measurement loop for Section~\ref{sec:research-agenda}.

\paragraph{Mapping to existing reliability and governance frameworks.}
The DMM is not a compliance standard, but the primitive it formalises
($\delta$, $\varepsilon$, $\mathrm{SCI}^{\delta}$) is precisely what
existing reliability and AI-governance regimes ask for under different
names. Three useful handshakes:
\begin{itemize}
  \item \emph{EU AI Act Article 15} requires high-risk AI systems to
        achieve ``appropriate levels of accuracy, robustness and
        cybersecurity'' over their lifecycle and to publish those
        metrics~\citep{eu2024aiact}. $\delta$ (chain-aware semantic
        consistency) and $\varepsilon$ (verifier KL gap) operationalise
        ``accuracy and robustness'' for agentic workflows in a way
        per-endpoint SLOs do not.
  \item \emph{NIST AI RMF MEASURE function} requires identification of
        reliability and validity metrics with measurement procedures.
        The DMM telemetry plane (Component C) is a concrete instantiation:
        $\delta$ per persona, $\varepsilon$ per verifier, $\mathrm{SCI}^{\delta}$
        per panel.
  \item \emph{ISO/IEC 42001 AI Management System} requires auditable
        operational telemetry. A $D3$/$D4$ platform's published
        determinism telemetry plane is audit-ready under this standard
        in a way a $D2$ platform's per-endpoint SLOs are not.
\end{itemize}
We do not propose the DMM \emph{as} a compliance standard; we propose
it as a substrate primitive that governance frameworks can adopt to
operationalise the chain-aware-reliability requirements they already
gesture at.

\paragraph{Boundary: what the DMM is not.}
The DMM is not a compliance standard, procurement requirement, or
substitute for domain-specific safety, privacy, or regulatory
regimes~\citep{anderljung2023frontier, eu2024aiact, bengio2024isr}.
It is a structural ladder for the deterministic-grounding axis;
orthogonal concerns (privacy, fairness, safety, liability) require
their own ladders. We also do not address whether higher DMM levels
should be \emph{mandated}; that is a governance question whose answer
is independent of the technical one this paper addresses.

\section{A Falsifiable Research Agenda}\label{sec:research-agenda}

A position becomes science only when stated as falsifiable questions. We
close the technical body with five open questions, each paired with a
measurement design and an explicit null result. They align with the
research clusters of~\citet{genewein2026agi}\footnote{Specifically:
quantitative forecasting (OQ1), benchmarking and abstraction (OQ2),
multi-agent scaling (OQ4), recursive improvement dynamics (OQ5),
governance and embodied dynamics (OQ3).}.

\begin{takeawaybox}{Box 8 $\vert$ The OQ1--OQ5 open-question programme.}
\textbf{OQ1 --- Data wall.} Does an economically self-sustaining,
verifier-equipped, deterministic real-supply environment measurably
outperform pure simulation or pure self-generated data in (a) sample
efficiency and (b) time-to-onset of recursive degeneration?

\textbf{OQ2 --- Abstraction barrier.} Can models trained on real-supply
multimodal distributions discover stable new conceptual primitives
beyond the existing CPV ontology? Does the SCI predict the share of
post-ontology stable concepts?

\textbf{OQ3 --- Embodied bottleneck.} How low can the latency floor of a
C2M physical loop be pushed at constant manufacturability quality?

\textbf{OQ4 --- Multi-agent scaling under trust and determinism.} How does
matching quality scale jointly with agent count, supply-trust coverage,
and environment determinism? Does $\mathrm{SCI}^{\delta}$ predict
downstream agentic capability across DMM levels $D2 \to D4$?

\textbf{OQ5 --- Verifier ceiling.} At what verifier-quality $\varepsilon$
does the data flywheel transition from accumulation to degeneration? At
what $\varepsilon$ does
Proposition~\ref{prop:gsec-formal} fail?
\end{takeawaybox}

\subsection{OQ1: data wall experiments}

\paragraph{Measurement design.}
A three-way training comparison: matched-budget runs of (a)
verifier-gated real-supply data from a $D3$+ environment, (b)
high-fidelity simulation, (c) state-of-the-art recursive
self-generation~\citep{li2023selfalignment, yuan2024self}. Hold model
architecture, total tokens, and verifier architecture constant.
Evaluate sample efficiency and recursive-degeneration onset following
the protocol of~\citet{shumailov2024ai}.

\paragraph{Null result.}
If (a) does not measurably beat (b) and (c) on either axis at the
matched-budget point, the privileged-grounding hypothesis is weakened.

\subsection{OQ2: abstraction-barrier experiments}

\paragraph{Measurement design.}
Train a multimodal model on the real-supply distribution; cluster latent
concepts that pass a stability test across random seeds; quantify the
share of stable concepts not previously represented in the curated CPV
ontology. Replicate against a baseline trained only on human-curated
catalogs.

\paragraph{Null result.}
If the share of post-human-ontology stable concepts is statistically
indistinguishable from baseline, the abstraction-barrier claim is
weakened.

\subsection{OQ3: embodied-bottleneck experiments}

\paragraph{Measurement design.}
Measure end-to-end design-to-first-pass-quote latency in a C2M pipeline
under (a) current best-of-class and (b) a pipeline augmented with an
AI-generated manufacturability validator. Compare against the analytic
bound of~\citet{lawrence2024atomic}.

\paragraph{Null result.}
If the latency floor is invariant to AI investment, the customisability
property fails to relax the embodied bottleneck.

\subsection{OQ4: multi-agent scaling under supply trust and DMM level}

\paragraph{Measurement design.}
In a real A2A matching market, vary $\tau_{\mathrm{supply}}$ and
$\delta_{\mathrm{env}}$ along their observable ranges, holding agent
population and protocol fixed; measure matching quality and clearing
rate. Fit $Q_{\mathrm{match}}(N, \rho, \tau, \delta)$ and the marginal
effect of each DMM step.

\paragraph{Null result.}
If $Q_{\mathrm{match}}$ is approximately independent of $\tau$ and
$\delta$, the trust-substrate and determinism roles are weakened. If
$\mathrm{SCI}^{\delta}$ does not rank-order platforms by agent-task
success, the SCI construct is weakened.

\subsection{OQ5: verifier-quality ceiling and flywheel degeneration}

\paragraph{Measurement design.}
Estimate verifier irreducible error $\varepsilon$ on a held-out
adversarial panel. Train a sequence of distilled models with the
verifier in the loop; identify the training-generation index at which
validation loss increases. Replicate at several $\varepsilon$ levels.
Measure the $R$-$R^\star$ gap against
Proposition~\ref{prop:goodhart-floor}'s Pinsker bound.

\paragraph{Null result.}
If the flywheel degenerates at $\varepsilon$ levels substantially below
production verifiers, the collapse-resistance argument is overstated.

\paragraph{Extension to environment-side evolution.}
A natural extension: does skill-library quality degrade faster, slower,
or at the same rate as model quality as verifier error increases?
Proposition~\ref{prop:gsec-formal} predicts the rates differ by
$\mu / d$; the experiment measures that factor empirically.

\subsection{Towards a non-saturating benchmark}

We propose a \emph{verifiable-environment benchmark} in which each
capability is scored by actual settlement, fulfilment, or
manufacturability outcome, following the trajectory
of~\citet{chollet2019measure, ho2025rosetta, liu2024agentbench,
zhou2023webarena, jimenez2024swebench, xie2024osworld,
yao2024taubench}. The benchmark is the public-good counterpart of
investment thesis (I5).

\subsection{A pilot measurement: $\tau$-bench under controlled $\delta$}\label{sec:pilot}

While the full OQ programme requires sustained research investment, a
first-pass test of the $\delta^k$ bound is achievable on existing
benchmarks. We sketch a lightweight protocol.

\paragraph{Setup.}
Take $\tau$-bench~\citep{yao2024taubench}, which evaluates multi-step
agent workflows against a retail back-end. Instrument the back-end's
tool-execution layer with a deterministic-seeded perturbation hook
that, with probability $1-\delta_{\mathrm{inj}}$ per tool call,
replaces the ground-truth API response with a semantically plausible
incorrect alternative. The \emph{injector} is template-driven (not
LLM-generated) to make replications byte-identical: a fixed dictionary
mapping each tool to a pool of canonical perturbations (stale
inventory $\to$ inventory$-1$ unit, price perturbation $\to$
$\pm 5\%$, attribute swap $\to$ swap with a randomly chosen sibling
SKU from the same category). The perturbation seed is fixed per
$(\text{task ID}, \text{trial ID}, \delta_{\mathrm{inj}})$ triple.

\paragraph{Protocol.}
Run the benchmark at $\delta_{\mathrm{inj}} \in \{1.0, 0.95, 0.9, 0.8,
0.7\}$ across three agent architectures (direct prompting, ReAct, and
a Reflexion-augmented agent~\citep{shinn2023reflexion}). For each
$(\delta_{\mathrm{inj}}, \text{agent})$ cell, record chain-task
success rate over $n \geq 200$ episodes, stratified by $\tau$-bench's
nominal chain-length $k$. Estimate the effective chain length
$k_{\mathrm{eff}}$ per agent architecture by maximum-likelihood fit of
the model $P_{\mathrm{success}} = \delta_{\mathrm{inj}}^{k_{\mathrm{eff}}}$
in logit space, with $k_{\mathrm{eff}}$ as the single free parameter
and 95\% confidence intervals via bootstrap over episodes. Report
inter-trial reproducibility (variance of success across 3 seed
replications per cell) as a robustness diagnostic.

\paragraph{Expected outcome.}
If the position is correct, the empirical success-rate curve should
track $\delta_{\mathrm{inj}}^{k_{\mathrm{eff}}}$ for an effective
chain length close to the benchmark's nominal chain length. Reflexion
should shift $k_{\mathrm{eff}}$ downward but not eliminate the
exponential shape (Remark~\ref{rem:retry}).

\paragraph{Null result.}
If, after the bootstrap-CI analysis, $\Pr[\text{success}]$ is
approximately invariant to $\delta_{\mathrm{inj}}$ across agents
(slope statistically indistinguishable from zero on the
logit-transformed curve), the $\delta^k$ bound does not bite in
practice and the determinism framing is empirically unsupported.

\paragraph{Power analysis.}
At $n = 200$ per cell and binary outcomes, the standard error on the
estimated success rate is bounded by $1/(2\sqrt{200}) \approx 0.035$.
The position predicts (at $k_{\mathrm{eff}} = 6$, the median nominal
$\tau$-bench chain length) success rates of $\{1.0, 0.74, 0.53, 0.26,
0.12\}$ across the five $\delta_{\mathrm{inj}}$ levels; adjacent-cell
gaps are $\{0.26, 0.21, 0.27, 0.14\}$, each substantially above
$2 \times 0.035 = 0.07$. The proposed $n$ is therefore adequate to
distinguish the position's prediction from a null at the 95\% level.

\subsection{Preliminary empirical grounding}\label{sec:preliminary}

The scale of verifiable signal in real economic environments is
already substantial. As one anchored data point for
\emph{settlement-class} signals: Stripe's public disclosures indicate
over a trillion dollars in 2024 payment volume across hundreds of
millions of transactions~\citep{stripe2024annual}, implying
$\sim\!10^6$--$10^7$ verifiable settlement events per day at that
single platform. We restrict this anchor to settlement-class signals:
Stripe is a payments processor and does not provide direct evidence
for sourcing, fulfilment, or quality-inspection throughputs, which are
the multimodal signal classes most relevant to the grounding
argument. Whether large B2B sourcing platforms reach comparable
daily throughput on the broader signal classes (order confirmations,
shipment verifications, quality inspections) is an empirical question
that OQ1's measurement design is exactly intended to answer.

For comparison: the largest curated RLHF preference datasets contain
$O(10^5)$ comparison pairs~\citep{ouyang2022instructgpt,
bai2022constitutional}; even scaled-up RLVR pipelines produce
$O(10^6)$ verification events per training
run~\citep{wang2024executable, dubey2024llama3}. Settlement-class
signal alone therefore provides at least a $10\times$ daily-throughput
advantage over curated alignment efforts; the multimodal
sourcing-class signal advantage is what OQ1 measures.

\paragraph{Signal-quality caveat.}
Throughput is not the only relevant axis: most settlement signals are
binary (cleared or returned) and not labelled at concept-discovery
granularity. The position's flywheel claim therefore rests on
quality-weighted, not raw, signal volume; OQ1 must control for label
density and concept-coverage per signal class, not just count signals.

Recent benchmarks have begun to operationalise this industrial
complexity: holistic e-commerce agent
evaluation~\citep{min2025ecombench}, industrial knowledge
boundaries~\citep{bai2026industrybench}, multi-image product
understanding for industrial
catalogues~\citep{qi2026industrybenchmipu}, and operator-class
long-horizon evaluation~\citep{xie2024osworld, yao2024taubench}. These
confirm that verifiable settlement and fulfilment signals can serve as
scalable ground-truth for agent evaluation. While individual platforms
are proprietary, the OQ programme can be instantiated on any supply
environment meeting conditions (i)--(v) at $D \geq 3$. Multiple
independent instantiations are expected.

\section{Counterarguments, Competing Positions, and Boundaries}\label{sec:counter}

A position is strongest when its boundaries are explicit. We split the
discussion into (a) \emph{objections} internal to the position and
(b) three \emph{competing positions} held by serious researchers that
predict different outcomes.

\subsection*{A. Objections internal to the position}

\subsection{Domain specificity}

\paragraph{Objection.}
Even granting that commercial supply is a privileged grounding substrate
for commerce-adjacent capabilities, does grounding generalise beyond
commerce?~\citep{morris2024levels, bostrom2014}.

\paragraph{Reply.}
We do not claim universality. The position is a capability-subset claim:
grounding-driven dominance for tasks whose verification is naturally
economic, physical, or multi-party. Purely linguistic or aesthetic tasks
are outside scope. Scientific verification has its own substrates; we
view this as future work.

\subsection{Adversariality of trustworthy fields}

\paragraph{Objection.}
Decision-grade fields have a long-tail distribution and are subject to
adversarial forgery~\citep{biggio2018wild, evtimov2018robust}.
Trustworthiness is a moving target.

\paragraph{Reply.}
We agree. The verifier must be continually hardened; OQ5's $\varepsilon$
is the formal quantity adversarial corruption increases. The position
weakens in proportion to the rate at which adversaries corrupt
decision-grade fields faster than verifier retraining.

\subsection{Incentive misalignment and systemic risk}

\paragraph{Objection.}
A strong trust substrate also amplifies incentive misalignment,
instrumental convergence, and systemic
risk~\citep{russell2019human, ngo2022alignment, hendrycks2023overview}.
Building the substrate without governance is irresponsible.

\paragraph{Reply.}
The position is conditional: build the substrate, and governance must
keep pace~\citep{anderljung2023frontier, eu2024aiact, bengio2024isr,
schuett2024thresholds}. Each investment thesis
(Section~\ref{sec:invest-theses}) concentrates capital \emph{and}
governance attention.

\subsection{Coordination risk}

\paragraph{Objection.}
The flywheel only turns when data standards are shared. Fragmented
schemas and platform fragmentation will erode
efficiency~\citep{lessig2006code, benkler2002coases}.

\paragraph{Reply.}
Coordination is a first-class friction. The argument is not that supply
grounding is free; it is that it pays back its coordination cost in
regimes where agents consume verifiable signal faster than human-curated
schemas produce it. The MCP / A2A protocol layer is precisely the
standardisation investment whose cost is being paid back.

\subsection{Data sovereignty and reproducibility}

\paragraph{Objection.}
Supply data is proprietary. If the position depends on data only platform
operators can access, it cannot be independently
verified.\footnote{A related objection concerns citation status: the
methodological-anchor citations in this paper
(see~\citealt{ding2026agenther, ding2026adarubric} \emph{inter alia})
are cited as support for specific claims --- trajectory recycling and
verifier adaptivity --- not as evidence that the position has been
validated. Only~\citet{ding2023efficiencyspectrum} is directly cited by
the AGI$\to$ASI survey~\citep{genewein2026agi} we extend.}

\paragraph{Reply.}
The position does not require data sharing between competitors. It
requires each participant to build a grounding flywheel on their own
supply data, climbing the DMM ladder to create a new form of platform
moat~\citep{coase1937nature}. For open science: (a) public benchmarks
(Section~\ref{sec:preliminary}); (b) the falsifiable form is ``any
environment satisfying (i)--(v) at $D \geq 3$ should exhibit the
predicted advantage''; (c) anonymised aggregate statistics can be
published without trade secrets.

\subsection{Stakeholder tensions the position creates}\label{sec:stakeholder-tensions}

The DMM and SCI implicitly side with agentic consumers in conflicts
that the supply environment must mediate. Four stakeholder
perspectives are in tension with the position as written:

\paragraph{End users / consumers.}
Personalisation, exploration injection, and engagement-optimised
ranking are user-welfare features for human-facing surfaces. Pushing
platforms to $D3$/$D4$ on user-facing surfaces would impose real
welfare costs on this constituency. The defensible position is
\emph{persona-conditioned} determinism: deterministic interfaces under
an agent persona, conventional engagement-optimised interfaces under
a human persona. The DMM should be read as a property of the
\emph{agent-persona surface}, not of the platform's user-facing
surface.

\paragraph{Suppliers / sellers on supply platforms.}
Smaller or newer suppliers benefit from exploration injection that
$D3$ removes; $D3$/$D4$ ranking with stable, faithful order
concentrates visibility on incumbents with established trust signals.
This is a real cost the position does not internalise; the SCI's
trustworthy-supply property treats supply-side coverage as a gain
without engaging the distributional effect on the supply side
itself. A future extension should model the supply-side incentive
implications of climbing the DMM.

\paragraph{Regulators and competition authorities.}
Deterministic substrates concentrate agent access into a few platforms
that have invested the capital to climb the DMM ladder. We celebrate
this as a ``platform moat'' (Section~\ref{sec:counter} reply to data
sovereignty) but the antitrust corollary is real: agent-mediated
commerce on $D3$/$D4$ substrates is a more concentrated market than
the equivalent on $D2$ substrates. The EU Digital Markets Act and
U.S.\ FTC merger guidelines on platform gatekeeping are directly
adjacent; the position has no answer for them.

\paragraph{Smaller platforms and developing-market marketplaces.}
DMM compliance has a capital cost (verifier service, skill registry,
telemetry plane). $D3$ is achievable for hyperscale incumbents and
may be out of reach for small or developing-market platforms. A world
stratified by DMM level concentrates flywheel-grade data into
incumbents who can train and serve foundation agents; smaller
platforms become data tributaries rather than independent flywheel
operators. This is a real concern even if the formal position is
correct, and we list it as a Limitations bullet
(Section~\ref{sec:conclusion}).

\subsection*{B. Strong competing positions}

We engage three competing positions. Each predicts different outcomes
and provides a route by which the position could lose.

\subsection{Competing position 1: sim-to-real suffices}\label{sec:cp-sim2real}

The strongest version of this position holds that photorealistic
simulation with domain randomisation~\citep{tobin2017domain,
akkaya2019solving}, co-evolving world
models~\citep{Hafner2020Dreamer, bruce2024genie, fang2025webevolver},
and zero-shot transfer will close the gap before deterministic real
environments become widespread, implying that marginal compute on
simulator fidelity dominates marginal compute on real-environment
determinism. Section~\ref{sec:sim2real-engagement} gives the
structural reply: simulators inherit a designer-ontology cap, a
verifier ground-truth shift, and a reward-design fragility that
qualifying real environments do not. The position loses if OQ1
returns null, and we commit to that null in advance.

\subsection{Competing position 2: better alignment training suffices}\label{sec:cp-rlhf}

\paragraph{Position.}
Better RLHF / RLAIF / weak-to-strong supervision plus cognitive
architectures~\citep{bai2022constitutional, ouyang2022instructgpt,
burns2024weaktostrong, yao2023react, shinn2023reflexion,
sumers2024coala} will absorb environmental noise. The environment can
stay stochastic.

\paragraph{What it predicts.}
Per-step success rates will climb fast enough that $\delta^k$ ceases
to bite before deterministic environments become widespread.

\paragraph{Our reply.}
This is the position we take most seriously.
Proposition~\ref{prop:det-bound} is a per-step argument: model
robustness reduces effective non-deterministic steps but does not change
the exponential shape (Remark~\ref{rem:retry}).
Proposition~\ref{prop:goodhart-floor} bounds how far alignment training
alone takes a policy under fixed-quality verifiers. Better training and
better environments are complements; the position loses if operator-class
deployments reach end-to-end success rates above $90\%$ on workflows of
chain length $k \geq 8$ in $D2$ environments --- the regime in which
Proposition~\ref{prop:det-bound} predicts below-$50\%$ success at
$\delta = 0.9$ even with a generous retry budget
(Remark~\ref{rem:retry}). We commit to retraction in that case.

\paragraph{Variant: orchestration-first frameworks.}
A practitioner-side variant of this position is held by orchestration-first
frameworks (LangChain, LangGraph, CrewAI, AutoGen): the answer is not
better training but \emph{better scaffolding} --- state machines, retry
policies, human-in-the-loop checkpoints, and typed memory --- wrapped
around the existing stochastic environment. This variant inherits the
same per-step counter-argument: scaffolding reduces \emph{effective}
$k$ via batching and checkpointing but does not change per-step
$\delta$. Where the orchestrator's own decision edges are themselves
stochastic (LLM-routed tool selection, model-as-judge gating), the
scaffolding introduces additional $\delta$-loaded steps and can
\emph{worsen} the chain-task budget. The position's prediction:
orchestration frameworks built on $D2$ environments will hit a ceiling
that the same orchestration on $D3$/$D4$ does not.

\subsection{Competing position 3: AI as normal technology}\label{sec:cp-normal}

\paragraph{Position.}
\citet{narayanan2024aiasnormal} argue that AI's deployment will resemble
prior general-purpose technologies, with adoption gated by institutional
and regulatory diffusion --- not sharp infrastructure inflections.

\paragraph{What it predicts.}
Agent-task success improves through gradual sociotechnical channels with
no discontinuous improvement at any infrastructure inflection point.

\paragraph{Our reply.}
We do not contest institutional adoption; we contest the absence of an
inflection. The $D2 \to D3$ transition is non-linear: below the
verifier-channel threshold the flywheel does not turn; above it, it
compounds. If OQ4 returns a smooth fit with no inflection, this
competing position is supported and ours is weakened.

\subsection{What would force retraction?}

\begin{itemize}
  \item OQ1 null AND OQ5 high-$\varepsilon$ $\Rightarrow$
        flywheel half invalidated.
  \item OQ2 null AND OQ3 null $\Rightarrow$
        grounding-as-leverage retracted to multi-agent-trust-only.
  \item Concurrent failure of OQ1--OQ4 $\Rightarrow$
        position comprehensively refuted.
  \item $\delta$ does not predict chain-task success across domains
        $\Rightarrow$ determinism framing abandoned.
  \item Operator-class deployments exceed $90\%$ end-to-end success on
        $k \geq 8$ workflows in $D2$ environments
        $\Rightarrow$ competing position 2 wins; retract.
  \item Smooth OQ4 fit consistent
        with~\citet{narayanan2024aiasnormal} $\Rightarrow$ competing
        position 3 wins; retract inflection claim.
\end{itemize}

\section{Conclusion}\label{sec:conclusion}

We have argued that the central frictions on AGI$\to$ASI progress
--- data wall, abstraction barrier, embodied bottleneck, multi-agent
trust --- are not first-order compute problems but \gls{grounding}
problems, and that commercially self-sustaining supply environments
satisfying a \emph{deterministic interface guarantee} form a privileged
grounding substrate. Three formal results anchor the argument:
\begin{itemize}
  \item the Determinism--Efficiency Bound
        (Proposition~\ref{prop:det-bound}): chain-task success degrades
        as $\delta^k$;
  \item the Verifier--Goodharting Floor
        (Proposition~\ref{prop:goodhart-floor}): flywheel asymptotics
        bounded by verifier quality $\varepsilon$;
  \item the Grounded Self-Evolution Convergence Condition
        (Proposition~\ref{prop:gsec-formal}): environment-side skill
        evolution accumulates iff $\delta > \delta_{\min}$.
\end{itemize}

\noindent We have decomposed \gls{supplycertainty} into five properties
aggregated into the Supply Certainty Index (Definition~\ref{def:sci}),
proposed a Determinism Maturity Model (Section~\ref{sec:dmm}) as an
adoption ladder, named five investable primitive categories
(Section~\ref{sec:invest-theses}), and committed to a falsifiable
strong form with five open questions whose null results would weaken or
retract the position. We have engaged explicitly with three strong
competing positions, each of which would force retraction of specific
claims.

\subsection*{A sustainability corollary}
Every failed chain-task is a form of wasted compute: an agent that
fails at step $j < k$ of a $k$-step plan consumes inference cost for
$j$ steps and produces nothing. The $\delta^k$ bound implies that
environment determinism is a major lever for compute efficiency in
agentic workloads at scale. A platform moving from $\delta = 0.8$
to $\delta = 0.95$ on $k = 6$ workflows shifts chain success from
$0.26$ to $0.74$, cutting wasted inference roughly $2.8\times$ before
any model-side optimisation.

\paragraph{Partial-equilibrium caveat.}
This is an inference-layer, holding-everything-else-constant
calculation; a full lifecycle accounting must include
(i) the upstream cost of building the higher-$\delta$ substrate
(verifier services, telemetry plane, dedicated agent-persona
infrastructure), (ii) retry-induced compute on failed chains
(Remark~\ref{rem:retry}: shared retry budgets reduce the multiplier),
and (iii) amortised training compute over the lifetime of agents
deployed against the substrate. We do not perform that full
life-cycle accounting here; we observe only that the inference-layer
multiplier is large enough that sustainability and capability arguments
\emph{point in the same direction} on the $D2 \to D3$ transition, and
leave the precise life-cycle decomposition as future work.

\subsection*{Limitations}
Four limitations the present paper does not resolve.
\emph{First}, all three formal results
(Propositions~\ref{prop:det-bound},~\ref{prop:goodhart-floor},~\ref{prop:gsec-formal})
depend on the verifier returning consistent verdicts on independent
trials; verifier corruption mid-deployment is bounded only at the
asymptotic $\varepsilon$ level and is not modelled as a time-varying
adversarial process.
\emph{Second}, the Supply Certainty Index is a composite designed for
platform comparability \emph{within} a domain, not for absolute
capability prediction \emph{across} domains; an
$\mathrm{SCI}^{\delta}(\mathcal{E}) = 0.8$ in one domain may correspond
to capabilities that differ substantially from $\mathrm{SCI}^{\delta} =
0.8$ in another. Cross-domain absolute comparison is out of scope.
\emph{Third}, the position is consistently positive-conditional:
\emph{if} the substrate is built, \emph{then} the flywheel turns. We
do not address the governance side --- \emph{whether} the substrate
should be built at every site --- leaving that to the governance
literature~\citep{anderljung2023frontier, bengio2024isr, eu2024aiact,
schuett2024thresholds, hendrycks2023overview}.
\emph{Fourth, the position has a power-concentration corollary it does
not internalise.} DMM compliance has a real capital cost (verifier
service, skill registry, telemetry plane); a world stratified by DMM
level will see flywheel-grade data accumulate in hyperscale incumbents
who can afford to climb, while smaller and developing-market platforms
become data tributaries. We celebrate this as a ``platform moat'' in
Section~\ref{sec:counter}'s reply to the data-sovereignty objection,
but the equity and antitrust corollaries deserve treatment in
follow-on work (Section~\ref{sec:stakeholder-tensions}).

\paragraph{Provocation.}
If the position is correct, the most undervalued asset on the way to
superintelligence is the existing infrastructure of verified
commercial transactions. These signals are produced as by-products of
ordinary economic activity, and they constitute a large reservoir of
grounding data that no amount of synthetic generation can substitute
for.

We expect the decisive competitive variable for agentic AI in the
coming years to be how much of the world a platform can verifiably
observe and expose to agents at $D3$ or above. Compute remains
necessary; it is not sufficient.

\section*{Acknowledgements}
We thank colleagues who reviewed earlier drafts of the argument
backbone for sharpening the falsifiability discipline of the open-
question programme and for pressing the engagement with competing
positions in Section~\ref{sec:counter}.

\section*{Use of AI Assistants}
Large-language-model assistants were used for prose tightening,
cross-reference checking, and bibliographic search. All technical
claims, formal results, and the falsifiable research programme are
the responsibility of the authors.

\appendix
\section{Appendix: Mapping Table and DMM Cross-Reference}\label{app:mapping}

Table~\ref{tab:mapping} summarises the grounding mapping argument of
Sections~\ref{sec:grounding-thesis}--\ref{sec:agent-econ} in a
single view. Columns list (i) the AGI$\to$ASI bottleneck, (ii) the
supply-certainty property that addresses it, (iii) a generic
operational proxy by which an implementation could be measured without
recourse to any organisation's internal targets, (iv) the falsifiable
measurement (one of OQ1--OQ5), (v) the suggested literature anchor, and
(vi) the DMM level (Section~\ref{sec:dmm}) at which the property
becomes load-bearing.

\begin{table}[ht!]
  \centering
  \footnotesize
  \begin{tabularx}{\textwidth}{p{2.0cm} p{1.8cm} X p{2.0cm} p{2.6cm} p{1.2cm}}
    \toprule
    \textbf{Bottleneck} & \textbf{Property} & \textbf{Generic operational proxy} & \textbf{Falsifiable measurement} & \textbf{Literature anchor} & \textbf{DMM} \\
    \midrule
    Data wall            & Thick         & Coverage of long-tail leaf categories; demand-to-supply translation recall; net new SKU rate                                                       & OQ1 (sample efficiency and collapse onset)        & \citet{shumailov2024ai,ding2023efficiencyspectrum,Villalobos2024WillRunOut}            & $\geq D3$ \\[3pt]
    Abstraction barrier  & Understandable & Same-item clustering precision/recall; attribute-governance coverage; verified knowledge-graph size                                              & OQ2 (post-ontology stable concept share)         & \citet{ortega2021shaking,ding2023parameter,wang2024mitigating}                          & $\geq D3$ \\[3pt]
    Embodied bottleneck  & Customisable  & Manufacturability rate of generated CAD artefacts; CAD-kernel coverage; spec-to-first-quote latency; factory-capacity matching accuracy           & OQ3 (latency-floor slope under AI investment)    & \citet{lawrence2024atomic,genewein2026agi}                                              & $\geq D3$ \\[3pt]
    Multi-agent trust    & Trustworthy   & Decision-grade-field coverage; verified-supply set size; independent judge precision/recall; field freshness                                       & OQ4 (matching scaling under $\tau_{\mathrm{supply}}$); OQ5 (verifier ceiling) & \citet{tomasev2025virtual,lambert2024rewardbench,zhang2024intention,burns2024weaktostrong}                    & $D3 \to D4$ \\[3pt]
    (Matching structure) & Comparable    & Cluster precision and Cov@k of same-item retrieval; independent judge precision/recall                                                            & OQ4 (matching quality scaling)                   & \citet{leibo2018malthusian,liu2025spiral, hong2024metagpt}                                               & $\geq D3$ \\
    \bottomrule
  \end{tabularx}
  \caption{Grounding mapping summary, extended with DMM-level
  attribution. The table is the appendix counterpart of
  Figure~\ref{fig:grounding-map} and is the load-bearing summary of
  the paper's positive content. The DMM column gives the minimum
  maturity level at which the corresponding property becomes
  load-bearing on the position; below that level, investment in the
  property does not pay back via
  Proposition~\ref{prop:gsec-formal}.}
  \label{tab:mapping}
\end{table}

\paragraph{How to read.}
Each row is the assertion that the listed bottleneck is best addressed
by the listed property, measurable along the listed proxy, with the
listed open question available as a refutation route, the listed
literature as the most direct anchor, and the listed DMM level as the
operational threshold below which investment does not pay back.

\subsection*{Cross-reference: propositions, definitions, and questions}

For convenience, the load-bearing formal artefacts of the paper are:

\begin{itemize}
  \item Definition~\ref{def:det-env} --- Deterministic Agentic Environment.
  \item Definition~\ref{def:sci} --- Supply Certainty Index.
  \item Proposition~\ref{prop:det-bound} --- Determinism--Efficiency Bound.
  \item Lemma~\ref{lem:correlated} --- Correlation reshapes but does not erase exponential degradation.
  \item Remark~\ref{rem:retry} --- Retries do not erase the bound.
  \item Proposition~\ref{prop:goodhart-floor} --- Verifier--Goodharting Floor.
  \item Proposition~\ref{prop:gsec-formal} --- Grounded Self-Evolution Convergence.
  \item Box 1 --- Falsifiable strong form.
  \item Box 2 --- Assumptions of Proposition~\ref{prop:det-bound}.
  \item Box 3 --- Assumptions of Proposition~\ref{prop:goodhart-floor}.
  \item Box 4 --- Five sufficiency conditions for a privileged grounding environment.
  \item Box 5 --- Grounded Self-Evolution Principle (qualitative).
  \item Box 6 --- Five investment-thesis categories.
  \item Box 7 --- Determinism Maturity Model levels $D0$--$D4$.
  \item Box 8 --- The OQ1--OQ5 open-question programme.
\end{itemize}

\subsection*{Notation table}\label{app:notation}

The paper uses several closely related symbols for environment
determinism and verifier quality. We collect them here for reference.

\begin{table}[!h]
  \centering
  \small
  \begin{tabular}{l p{4.5cm} l p{4.0cm}}
    \toprule
    \textbf{Symbol} & \textbf{Meaning} & \textbf{First use} & \textbf{Notes} \\
    \midrule
    $\delta$, $\delta(\mathcal{E})$
      & Per-step success probability of environment $\mathcal{E}$ against ground truth
      & Box 2 (A1)
      & The general environment-level quantity \\[3pt]
    $\delta_i$
      & Per-step success probability of step $i$ in a chain
      & Lemma~\ref{lem:correlated}
      & Allows steps to differ \\[3pt]
    $\delta_i(s)$
      & Per-step success conditional on session state $s$
      & Lemma~\ref{lem:correlated}
      & Random variable in $s$ \\[3pt]
    $\bar\delta_i$
      & $:= \mathbb{E}_s[\delta_i(s)]$, marginal per-step determinism
      & Lemma~\ref{lem:correlated}
      & Marginal expectation \\[3pt]
    $\delta_{\mathrm{env}}$
      & Environment-aggregate determinism in $Q_{\mathrm{match}}$
      & §\ref{sec:agent-econ}
      & Same as $\delta(\mathcal{E})$ \\[3pt]
    $\delta_{\mathrm{inj}}$
      & Injected determinism in the $\tau$-bench pilot
      & §\ref{sec:pilot}
      & Experimentally controlled \\[3pt]
    $\delta_{\min}$, $\delta_{\max}$
      & Convergence thresholds for environment-side evolution
      & Box 5, Prop.~\ref{prop:gsec-formal}
      & Defined by sufficient condition \\[3pt]
    $\varepsilon$
      & $D_{\mathrm{KL}}(V \,\|\, V^\star)$, verifier KL gap
      & Box 3 (B1)
      & Verifier-quality measure \\[3pt]
    $\varepsilon_{\max}$
      & Verifier-quality threshold for monotone evolution
      & Box 5
      & Defined by sufficient condition \\[3pt]
    $C$
      & $\|R\|_\infty$, reward sup-norm bound
      & Box 3 (B3)
      & Replaces v8/v9's $L$-Lipschitz \\[3pt]
    $\tau_{\mathrm{supply}}$
      & Decision-grade-field coverage rate
      & §\ref{sec:agent-econ}
      & Supply-trust scalar \\[3pt]
    $\mathrm{SCI}(\mathcal{E};D)$
      & Supply Certainty Index, applicable subset $P_D$
      & Def.~\ref{def:sci}
      & $[0,1]$ \\[3pt]
    $\mathrm{SCI}^{\delta}$
      & $\delta$-corrected SCI
      & Def.~\ref{def:sci}
      & $:= \delta \cdot \mathrm{SCI}$ \\[3pt]
    $k$
      & Chain length
      & §\ref{sec:intro}, Prop.~\ref{prop:det-bound}
      & Number of steps \\[3pt]
    $k_{\mathrm{eff}}$
      & Effective chain length (post-retry)
      & §\ref{sec:pilot}
      & Fit parameter \\[3pt]
    $r$, $B$
      & Per-step retry budget; total retry budget
      & Remark~\ref{rem:retry}
      & $r = B/k$ under shared budget; $B$ rather than $R$ to avoid clash with reward~$R$ \\
    \bottomrule
  \end{tabular}
  \caption{Notation used in the paper, with first-use location and
  short definition.}
  \label{tab:notation}
\end{table}

\paragraph{What is deliberately absent.}
This appendix does not contain quantitative targets from any specific
organisation. The position is intended to be a generic structural
claim testable in any sufficiently rich commercial supply environment,
not a report on a particular system. Readers who need illustrative
numbers are referred to the public sustainability reports and
disclosures of major sourcing platforms.

\section*{Glossary}\addcontentsline{toc}{section}{Glossary}
\begin{theglossary}\glossaryheader
\glsgroupheading{A}\relax \glsresetentrylist %
\glossentry{abstractionbarrier}{\glossaryentrynumbers{\relax 
		\setentrycounter[]{page}\glsnumberformat{1}}}%
\glossentry{a2a}{\glossaryentrynumbers{\relax 
		\setentrycounter[]{page}\glsnumberformat{5}}}\glsgroupskip
\glsgroupheading{C}\relax \glsresetentrylist %
\glossentry{cpv}{\glossaryentrynumbers{\relax 
		\setentrycounter[]{page}\glsnumberformat{12}}}%
\glossentry{c2m}{\glossaryentrynumbers{\relax 
		\setentrycounter[]{page}\glsnumberformat{5}}}\glsgroupskip
\glsgroupheading{D}\relax \glsresetentrylist %
\glossentry{dataflywheel}{\glossaryentrynumbers{\relax 
		\setentrycounter[]{page}\glsnumberformat{5}}}%
\glossentry{datawall}{\glossaryentrynumbers{\relax 
		\setentrycounter[]{page}\glsnumberformat{1}}}%
\glossentry{deterministicenv}{\glossaryentrynumbers{\relax 
		\setentrycounter[]{page}\glsnumberformat{2}}}\glsgroupskip
\glsgroupheading{E}\relax \glsresetentrylist %
\glossentry{effectivecompute}{\glossaryentrynumbers{\relax 
		\setentrycounter[]{page}\glsnumberformat{1}}}%
\glossentry{embodiedbottleneck}{\glossaryentrynumbers{\relax 
		\setentrycounter[]{page}\glsnumberformat{1}}}\glsgroupskip
\glsgroupheading{G}\relax \glsresetentrylist %
\glossentry{groundedscaling}{\glossaryentrynumbers{\relax 
		\setentrycounter[]{page}\glsnumberformat{5}}}%
\glossentry{grounding}{\glossaryentrynumbers{\relax 
		\setentrycounter[]{page}\glsnumberformat{5}}}\glsgroupskip
\glsgroupheading{M}\relax \glsresetentrylist %
\glossentry{modelcollapse}{\glossaryentrynumbers{\relax 
		\setentrycounter[]{page}\glsnumberformat{5}}}%
\glossentry{multiagenttrust}{\glossaryentrynumbers{\relax 
		\setentrycounter[]{page}\glsnumberformat{1}}}\glsgroupskip
\glsgroupheading{S}\relax \glsresetentrylist %
\glossentry{supplycertainty}{\glossaryentrynumbers{\relax 
		\setentrycounter[]{page}\glsnumberformat{5}}}\glsgroupskip
\glsgroupheading{V}\relax \glsresetentrylist %
\glossentry{verifier}{\glossaryentrynumbers{\relax 
		\setentrycounter[]{page}\glsnumberformat{5}}}%
\end{theglossary}

\end{document}